\documentclass[screen]{jair}
\usepackage[utf8]{inputenc}

\setcopyright{cc}
\copyrightyear{2025}
\acmYear{2025}
\acmDOI{10.1613/jair.1.19171}

\JAIRAE{Marek Petrik}
\JAIRTrack{}
\acmVolume{ 83}
\acmArticle{33}
\acmMonth{8}
\acmYear{2025}

\RequirePackage[
  datamodel=acmdatamodel,
  style=acmnumeric,
  backend=biber,
  giveninits=true
  ]{biblatex}

\newcommand{\citep}{\cite}
\newcommand{\citet}{\textcite}

\usepackage{graphicx}
\usepackage{hyperref}
\usepackage{subfigure}
\usepackage{booktabs} 

\usepackage{url}

\usepackage{amsfonts}

\usepackage{thm-restate}
\usepackage[capitalise,noabbrev]{cleveref} 
\usepackage{diagbox}
\usepackage{multirow}
\usepackage{wrapfig}
\usepackage{algorithm}
\usepackage[noend]{algorithmic}
\usepackage{setspace}
\usepackage{anyfontsize} 

\theoremstyle{plain}
\newtheorem{theorem}{Theorem}[section]
\newtheorem{proposition}[theorem]{Proposition}

\theoremstyle{definition}

\newtheorem{assumption}[theorem]{Assumption}
\theoremstyle{remark}

\usepackage[textsize=tiny]{todonotes}

\newcommand{\pinv}{\dagger}

\newcommand{\figwidththree}{0.30\textwidth}

\input{definitions}

\begin{document}

\title[An MRP Formulation for Supervised Learning: Generalized Temporal Difference Learning Models]{An MRP Formulation for Supervised Learning: Generalized Temporal Difference Learning Models}

\author{Yangchen Pan}
\authornote{Corresponding Author.}
\orcid{0009-0000-8297-9045}
\email{yangchen.pan@eng.ox.ac.uk}
\affiliation{%
  \institution{University of Oxford}
  \city{Oxford}
  \country{United Kingdom}
}

\author{Junfeng Wen}
\authornote{Co-first \& Corresponding author with Yangchen Pan.}
\orcid{0000-0002-4128-1422}
\email{junfengwen@gmail.com}
\affiliation{%
  \institution{Carleton University}
  \city{Ottawa}
  \country{Canada}
}

\author{Chenjun Xiao}
\orcid{0009-0006-5051-7384}
\email{chenjunx@cuhk.edu.cn}
\affiliation{%
  \institution{The Chinese University of Hong Kong (Shenzhen)}
  \city{Shenzhen}
  \country{China}
}

\author{Philip H.S. Torr}
\orcid{0009-0006-0259-5732}
\email{philip.torr@eng.ox.ac.uk}
\affiliation{%
  \institution{University of Oxford}
  \city{Oxford}
  \country{United Kingdom}
}

\renewcommand{\shortauthors}{Pan, Wen, Xiao \& Torr}

\begin{abstract}
{\bf Background:}
Traditional supervised learning (SL) assumes data points are independently and identically distributed (i.i.d.), which overlooks dependencies in real-world data. Reinforcement learning (RL), in contrast, models dependencies through state transitions. 

{\bf Objectives:} 
This study aims to bridge SL and RL by reformulating SL problems as RL tasks, enabling the application of RL techniques to a wider range of SL scenarios. We aim to
model SL data as interconnected, and develop novel temporal difference (TD) algorithms that can accommodate diverse data types. Our objectives are to (1) establish conditions where TD outperforms ordinary least squares (OLS), (2) provide convergence guarantees for the generalized TD algorithm, and (3) validate the approach empirically using synthetic and real-world datasets.

{\bf Methods:} We reformulate traditional SL as a RL problem by modeling data points as a Markov Reward Process (MRP). We then introduce a concept analogous to the inverse link function in generalized linear models, allowing our TD algorithm to handle various data types. Our analysis, grounded in variance estimation, identifies conditions where TD outperforms OLS. We establish a convergence guarantee by conceptualizing the TD update rule as a generalized Bellman operator. Empirical validation begins with synthetic data progressively matching theoretical assumptions to verify our analysis, followed by evaluations on real-world datasets to demonstrate practical utility.

{\bf Results:}
Our theoretical analysis shows that TD can outperform OLS in estimation accuracy when data noise is correlated. Our approach generalizes across various loss functions and SL datasets. We prove that the Bellman operator in our TD framework is a contraction, ensuring convergence for both expected and stochastic TD updates. Empirically, TD outperforms SL baselines when data aligns with its assumptions, remains competitive across diverse datasets, and is robust to hyperparameter choices.

{\bf Conclusions:}
This study demonstrates that SL can be reformulated as a problem of interconnected data modeled by an MRP, effectively solved using TD learning. Our generalized TD is theoretically sound, with convergence guarantees, and practically effective. It generalizes OLS, offering superior performance on correlated data. This work enables RL techniques to benefit SL tasks, offering a pathway for future advancements.
\end{abstract}

\received{15 May 2025}
\received[accepted]{14 July 2025}

\maketitle

\section{Introduction}
\label{sec:intro}
The primary objective of statistical supervised learning (SL) is to learn the relationship between the features and the output (response) variable. 
To achieve this, generalized linear models \cite{nelder1972glm,mccullagh1989generalized} are considered a generic algorithmic framework employed to derive objective functions. These models make specific assumptions regarding the conditional distribution of the response variable given input features, which can take forms such as Gaussian (resulting in ordinary least squares), Poisson (resulting in Poisson regression), or multinomial (resulting in logistic regression or multiclass softmax cross-entropy loss).

In recent years, reinforcement learning (RL)—widely used in interactive learning scenarios—has experienced a surge in popularity. This growth has fostered increasing synergy between RL and SL, where each paradigm complements the other in meaningful ways. In SL-assisted RL, for example, the subfield of imitation learning\citep{ahmed2017il} leverages expert demonstrations to regularize or accelerate the RL training process. Weakly supervised methods\citep{lisa2020wsrl} have also been employed to constrain the RL task space, while relabeling techniques have facilitated more effective goal-conditioned policy learning~\citep{ghosh2021learning}. 
Return-conditioned policy learning naturally incorporates traditional SL techniques. A common approach is to maximize the log-likelihood of a policy conditioned not only on the state but also on the trajectory return~\citep{brandfonbrener2022rcsl}. This method is typically applied in offline RL settings, where the presence of high-return trajectories makes learning a return-conditioned policy particularly appealing.
Earlier work has also drawn from SL techniques to address RL challenges. For instance, \citet{tom2001vfabysvm} proposed using support vector machines for batch value function approximation. Similarly, kernel-based methods originally developed for SL were adapted to RL problems by \citet{dirk2002kbrl,andrew2016kbrl}.

Conversely, RL has also expanded its application into traditional SL domains. RL has proven effective in fine-tuning large language models~\citep{james2017rlhf}, aligning them with user preferences. Additionally, RL algorithms~\citep{gupta2021nn} have been tailored for training neural networks (NNs), treating individual network nodes as RL agents. In the realm of imbalanced classification, RL-based control algorithms have been developed, where predictions correspond to actions, rewards are based on heuristic correctness criteria, and episodes conclude upon incorrect predictions within minority classes~\citep{enlu2020imb}. Permative prediction~\citep{juan2020perfpred} provides a theoretical framework that can deal with nonstationary data and it can be reframed within a RL context. 

Existing approaches that propose a RL framework for solving SL problems often exhibit a heuristic nature. These approaches involve crafting specific formulations, including elements like agents, reward functions, action spaces, and termination conditions, based on intuitive reasoning tailored to particular tasks. Consequently, they lack generality, and their heuristic nature leaves theoretical assumptions, connections between optimal RL and SL solutions, and convergence properties unclear. 

To the best of our knowledge, among existing work, the most generic and theoretically sound approach to converting a SL problem into a RL problem involves using the loss function of the SL task to define a reward function---typically mapping smaller losses to higher rewards~\citep{andrew2004slrl}. However, this approach (1) does not fundamentally alter the underlying data generation assumption to that of a Markov process, and (2) consequently does not enable the application of one of RL’s most distinctive algorithms—temporal-difference (TD) learning—to supervised learning problems. Therefore, it remains unclear whether there exists a unified and systematic RL formulation that both redefines the data generation process and enables the use of TD learning to model a broad class of conventional SL tasks.
Such a formulation should be agnostic to learning settings, including various tasks such as ordinary least squares regression, Poisson regression, binary or multi-class classification, etc. 

In this study, we introduce a generic Markov process formulation for data generation, offering an alternative to the conventional i.i.d.\ data assumption in SL. Specifically, when faced with a SL dataset, we view the data points as originating from a Markov reward process (MRP)~\citep{szepesvari2010algorithms}. 
To accommodate a wide range of problems, such as Poisson regression, binary or multi-class classification, we introduce a generalized TD learning model in \cref{sec:method}. \cref{sec:beliefcompare} explores the relationship between the solutions obtained through TD learning and the original linear regression. 
Furthermore, we prove that under specific conditions with correlated noise, TD estimator is more efficient than the traditional ordinary least squares (OLS) estimator. 
We show that our generalized TD algorithm---which can handle various types of datasets with different loss functions---corresponds to the application of a specialized and generalized Bellman operator. We further prove that this operator is contractive and admits a unique fixed point. Based on this property, we establish convergence results under linear function approximation, as detailed in \cref{sec:theory}.
Our paper concludes with an empirical evaluation of our TD algorithm in \cref{sec:experiment}, verifying our theoretical results, assessing its critical design choices and practical utility when integrated with a deep neural network across various tasks,
achieving competitive results and, in some cases, improvements in generalization performance.
We view our work as a step towards unifying diverse learning tasks from two pivotal domains within a single, coherent theoretical framework.

\section{Background}
\label{sec:bg}
This section provides a brief overview of the relevant concepts from statistical SL and RL settings, laying the groundwork for the rest of this paper.

\subsection{Conventional Supervised Learning}

In the context of statistical learning, we make the assumption that data points, in the form of $(\xvec, y)\in \mathcal{X} \times \mathcal{Y}$, are independently and identically distributed (i.i.d.) according to some unknown probability distribution $P$.
The goal is to find the relationship between the feature $\xvec$ and the response variable $y$ given a training dataset $\mathcal{D}=\{(\xvec_i, y_i)\}_{i=1}^n$.

In a simple linear function approximation case, a commonly seen algorithm is ordinary least squares (OLS) that optimizes the squared error objective function
\begin{equation}\label{stat-regr}
\min_\vw\ ||\Xmat \vw - \vy||^2_2.
\end{equation}
where $\Xmat$ is the $n\times d$ feature matrix, $\vy$ is the corresponding $n$-dimensional training label vector, and the $\wvec$ is the parameter vector we aim to optimize.

From a probabilistic perspective, this objective function can be derived by assuming $p(y|\vx)$ follows a Gaussian distribution with mean $\xvec^\top \wvec$ and conducting 
maximum likelihood estimation (MLE) for $\wvec$ with the training dataset. 
It is well known that $\EE[Y | \vx]$ is the optimal predictor~\citep{Bishop06}. 
For many other choices of distribution $p(y | \vx)$, generalized linear models (GLMs)~\citep{nelder1972glm} are commonly employed to estimate $\EE[Y | \vx]$. 
This includes OLS, Poisson regression~\citep{nelder1974poisson} and logistic regression, etc.

An important concept in GLMs is the \emph{inverse link function}, which we denote as $f$, that establishes a connection between the linear prediction (also called the \textbf{logit}), and the conditional expectation: $\EE[Y|\vx] = f(\vx^\top \vw)$. 
For example, in logistic regression, the inverse link function is the sigmoid function. 
We later propose generalized TD learning models within the framework of RL that correspond to GLMs, enabling us to handle a wide range of data that are modeled by different distributions in supervised learning. 

\subsection{Reinforcement Learning}

Reinforcement learning is often formulated within the Markov decision process (MDP) framework. 
An MDP can be represented as a tuple $(\mathcal{S}, \mathcal{A}, r, P, \gamma)$~\citep{puterman2014markov}, where $\mathcal{S}$ is the state space, $\mathcal{A}$ is the action space, $r: \Scal\times\Acal \mapsto \mathbb{R}$ is the reward function, $P(\cdot|s,a)$ defines the transition probability, and $\gamma\in(0, 1)$ is the discount factor. 
Given a policy $\pi: \mathcal{S}\times\mathcal{A}\rightarrow [0, 1]$, the return at time step $t$ is $G_t=\sum_{i=0}^\infty \gamma^i r(S_{t+i},A_{t+i})$, and the value of a state $s\in \mathcal{S}$ is the expected return starting from that state $v^\pi (s) = \EE_\pi[G_t|S_t=s]$. 
In this work, we focus on the policy evaluation problem for a fixed policy, thus the MDP can be reduced to a Markov reward process (MRP)~\citep{szepesvari2010algorithms} described by $(\mathcal{S}, r^\pi, P^\pi, \gamma)$ where $r^\pi(s)\defeq \sum_a\pi(a|s)r(s,a)$ and $P^\pi(s'|s)\defeq \sum_a \pi(a|s) P(s'|s, a)$.
When it is clear from the context, we will slightly abuse notations and ignore the superscript $\pi$. 

In policy evaluation problem, the objective is to estimate the state value function of a fixed policy $\pi$ by using the trajectory $s_0, r_1, s_1, r_2, s_2, ... $ generated from $\pi$. 
In linear function approximation, the value function is approximated by a parametrized function $v(s)\approx \vx(s)^\top \wvec$ with parameters $\vw$ and some fixed feature mapping $\vx: \mathcal{S} \mapsto \mathbb{R}^d$ where $d$ is the feature dimension. 
Note that the state value satisfies the Bellman equation
\begin{equation}\label{bellmaneq}
    v(s) = r(s)+\gamma \EE_{S'\sim P(\cdot|s)}[v(S')].
\end{equation}
One fundamental approach to the evaluation problem is the temporal difference (TD) learning~\citep{sutton1988td}, which uses a sampled transition $s_t, r_{t+1}, s_{t+1}$ to update the parameters $\vw$ through stochastic fixed-point iteration based on (\ref{bellmaneq}) with a step-size $\alpha>0$: 
\begin{equation}\label{eq:rl-td-update}
\vw \gets \vw + \alpha (y_{t,td} - \vx(s_t)^\top \wvec) \vx(s_t),
\end{equation}
where $y_{t,td} \defeq r_{t+1} + \gamma \vx(s_{t+1})^\top \wvec$. 
To simplify notations and align concepts, we will use $\vx_{t}\defeq\vx(s_t)$. 
In linear function approximation setting, TD converges to the solution that solves the system $\Amat \wvec = \bvec$~\citep{steven1996lstd,tsitsiklis1997analysis}, where
\begin{align}
\Amat = \EE[\vx_t (\vx_t - \gamma \vx_{t+1})^\top ]
=\Xmat^\top \Dmat(\Imat-\gamma \Pmat)\Xmat
\qquad\text{and}\qquad
\vb = \EE[r_{t+1} \vx_t]
=\Xmat^\top\Dmat\vr
\label{eq:A_b_def}
\end{align}
with $\Xmat\in\RR^{|\Scal|\times d}$ being the feature matrix whose rows are the state features $\vx_t$, $\Dmat\in\RR^{|\Scal|\times |\Scal|}$ being the diagonal matrix with the stationary distribution $D(s_t)$ on the diagonal, $\Pmat\in\RR^{|\Scal|\times |\Scal|}$ being the transition probability matrix (i.e., $\Pmat_{ij}=P(s_j|s_i)$) and $\vr\in\RR^{|\Scal|}$ being the reward vector.
Note that the matrix $\Amat$ is often invertible under mild conditions~\citep{tsitsiklis1997analysis}.

\section{MRP View and Generalized TD Learning}
\label{sec:method}
This section describes our MRP construction given the same dataset $\mathcal{D} = \{(\vx_i, y_i)\}_{i=1}^n$ and proposes our generalized TD learning algorithm to solve it. 
This approach is based on the belief that these data points originate from some MRP, rather than being i.i.d.\ generated.

\textbf{Regression}.
We start by considering the basic regression setting with linear models before introducing our generalized TD algorithm. 
\cref{tab:slmdp} summarizes how we can view concepts in conventional SL from an RL perspective. 
The key is to treat the original training label as a state value that we are trying to learn, and then the reward function can be derived from the Bellman equation (\ref{bellmaneq}) as
\begin{align}
\label{eq:reward_def}
r(s) = v (s) - \gamma \EE_{S'\sim P(\cdot|s)}[v(S')].
\end{align}
We will discuss the choice of $\Pmat$ later. 
Note that although \cref{eq:reward_def} may appear familiar to those in the inverse reinforcement learning (IRL) community, where the motivation often stems from the difficulty of specifying a reward function and thus inferring it from expert demonstrations, our approach is fundamentally different from IRL. 
In our case, the motivation is not related to the challenge of reward specification, as this is a supervised learning problem in which even the value function is explicitly known. 
As a result, the reward can be directly estimated from a single transition.
At each iteration (or time step in RL), the reward can be approximated using a stochastic example. 
For instance, assume that at iteration $t$ (i.e., time step in RL), we obtain an example $(\vx^{(t)}_i, y^{(t)}_i)$. 
We use superscripts and subscripts to denote that the $i$th training example is sampled at the $t$th time step. 
The next example $(\vx^{(t+1)}_j, y^{(t+1)}_j)$ is then sampled according to
$P(\cdot|\vx^{(t)}_i)$ and the reward can be estimated as $r^{(t+1)} = y^{(t)}_{i} - \gamma y^{(t+1)}_{j}$ by approximating the expectation in \cref{eq:reward_def} with a stochastic example.
As one might notice, in a sequential setting, $t$ is monotonically increasing, so from now on, we will simply use the simplified notation $(\vx_t, y_t)$ to denote the training example sampled at time step $t$. 

\begin{table}[t]
    \caption{How definitions in SL correspond to those in RL}
    \label{tab:slmdp}
    \centering
    \begin{tabular}{c|c}
        \hline
        \textbf{SL Definitions} & \textbf{RL Definitions} \\
        \hline
        Feature matrix $\Xmat$ & Feature matrix $\Xmat$ \\
        \hline
        Feature of the $i$th example $\xvec_i$ & The state feature $\vx(s_i) = \xvec_i$ \\
        \hline
        Training target $y_i$ of $\xvec_i$ & The state value $v(s_i)=y_i$ \\
        \hline
    \end{tabular}
\end{table}

We now summarize and compare the updating rules in conventional SL and in our TD algorithm under linear function approximation. 
At time step $t$, the \textbf{conventional updating rule} based on stochastic gradient descent (SGD) is
\begin{align}\label{eq:sl-lr}
    \vw &\gets \vw + \alpha (y_t - \vx_t^\top \vw) \vx_t, 
\end{align}
while our \textbf{TD updating rule} is
\begin{align}\label{td-lr}
    \vw &\gets \vw + \alpha (y_{t,td} - \vx_t^\top \vw) \vx_t
    \\
    \text{where }\ 
    y_{t,td} &\defeq r_{t+1} + \gamma \textcolor{red}{\widehat{y}_{t+1}}
    =  y_t - \gamma y_{t+1} 
    + \gamma \textcolor{red}{\vx_{t+1}^\top \vw} 
    \label{eq:td-target}
\end{align}
and $\vx_{t+1} \sim P(\cdot|\vx_t)$ with the ground truth label $y_{t+1}$.
The critical difference is that TD uses a bootstrap, so it does not cancel the $\gamma y_{t+1}$ term from the reward when computing the TD training target $y_{t,td}$. 
By setting $\gamma=0$, the original supervised learning updating rule (\ref{eq:sl-lr}) is recovered. 

\textbf{Generalized TD: An extension to general learning tasks}. 
A natural question regarding TD is how to extend it to different types of data, such as those with counting, binary, or multiclass labels. 
Recall that in generalized linear models~(GLMs), it is assumed that the output variable $y\in\Ycal$ follows an exponential family distribution. 
In addition, there exists an \emph{inverse link function} $f$ that maps a linear prediction $z\defeq\vw^\top\vx$ to the output/label space $\Ycal$ (i.e., $f(z)\in\Ycal,\forall z\in\RR$). 
Examples of GLMs include linear regression (where $y$ follows a Gaussian distribution, $f$ is the identity function and the loss is the squared loss) and logistic regression (where $y$ is Bernoulli, $f$ is the sigmoid function and the loss is the log loss). 
More generally, the output may be in higher-dimensional space and both $z$ and $y$ will be vectors instead of scalars. 
As an example, multinomial regression uses the softmax function $f$ to convert a vector $\vz$ to another vector $\vy$ in the probability simplex. 
Interested readers can refer to \citet{banerjee2005clustering,helmbold1999relative}; \citet[Table~2.1]{mccullagh1989generalized} for more details. 
As per convention, we refer to $z$ as \textbf{logit}. 

Returning to the TD algorithm, the significance of logit $z$ is that it is naturally \emph{additive on the real line} (e.g., in logistic regression, a larger logit indicates a higher chance of being in the positive class), which mirrors the additive nature of returns (cumulative sum of rewards) in RL. 
It also implies that two linear predictions can be added and the resultant $z=z_1+z_2$ can still be transformed to a valid output $f(z)\in\Ycal$. 
In contrast, adding two labels does not necessarily produce a valid label in general $y_1+y_2\notin\Ycal$. 
Therefore, the idea is to construct a bootstrapped target in the real line (logit space, or $z$-space)
\begin{align}
z_{t,td}
&\defeq r_{t+1} + \gamma \widehat{z}_{t+1} = (z_t - \gamma z_{t+1}) + \gamma \xvec_{t+1}^\top \wvec
\end{align}
and then convert it back to the original label space to obtain the TD target $y_{t,td}=f(z_{t,td})$. 
In multiclass classification problems, we often use a one-hot vector to represent the original training target. 
For instance, in the case of MNIST~\citep{mnist}, the target is a ten-dimensional one-hot vector. 
Consequently, the reward becomes a vector with each component corresponding to a class. 
This can be interpreted as evaluating the policy under ten different reward functions in parallel and selecting the highest value for prediction.

\cref{alg_gltd} provides the pseudo-code of our algorithm when using linear models. 
At time step $t$, the process begins by sampling the state $\vx_t$, and then we sample the next state according to the predefined $P$. 
The reward is computed as the difference in logits after converting the original labels $y_t,y_{t+1}$ into the logit space with the link function. 
Subsequently, the TD bootstrap target is constructed in the logit space. 
Finally, the TD target is transformed back to the original label space before it is used to calculate the loss.
Note that the standard regression is a special case where the (inverse) link function is simply the identity function, so it reduces to the standard update (\ref{eq:sl-lr}) with squared loss. 
In practice, we might need some smoothing parameter when the function $f^{-1}$ goes to infinity. For example, in binary classification, $\Ycal=\{0, 1\}$ and the corresponding logits are $z=-\infty$ and $z=\infty$. 
To avoid this, we subtract/add a small value to the label before applying $f^{-1}$ so that $z$ is finite. 
In case of non-linear models such as deep neural networks (DNNs), the update to the parameters $\vw$ can be carried out by gradient descent on the corresponding loss function, replacing $y_t$ with $f(z_{t,td})$.

\begin{algorithm}
\setstretch{1.1}
\setlength{\textfloatsep}{0pt}
  \begin{algorithmic}\caption{Generalized TD for SL}
  \label{alg_gltd}
  \STATE \textbf{Input:} A dataset $\Dcal$ 
  \STATE Initialize $\vw_0$
  \STATE Randomly sample a starting data point $(\vx_0, y_0)\in \mathcal{D}$. (One can also use mini-batch starting points in NNs.)
  \FOR {$t = 0, 1, 2, \ldots$}
    \STATE Sample $\vx_{t+1} \sim P(\cdot|\vx_t)$, let $y_{t+1}$ be its label
    \STATE $r_{t+1} = f^{-1}(y_t) - \gamma f^{-1}(y_{t+1})$ \textcolor{red}{// $f^{-1}$ converts label to logits}
    \STATE $z_{t,td} = r_{t+1} + \gamma \vx_{t+1}^\top \vw_t$ \textcolor{red}{// Bootstrap target, a separate target network is needed in DNNs}
    \STATE 
    $\vw_{t+1} \gets \vw_t - \alpha (f(\vx_t^\top\vw_t)-f(z_{t,td})) \xvec_t$
    \textcolor{red}{// $f$ converts logits back to label}  
  \ENDFOR
  \end{algorithmic}
\setlength{\textfloatsep}{0pt}
\end{algorithm}

We use the term \emph{generalized} as it corresponds to generalized linear models in SL. 
This usage should be distinguished from a similarly named concept in \citet{hado2019generalbellman} within the realm of RL. 
Our approach applies RL to address problems typically considered within the scope of SL. 
This integration is non-trivial and unconventional, as SL is commonly viewed as a one-shot prediction problem. 
Instead, we have introduced a sequential perspective to SL. 
Moreover, the function $f(\cdot)$ in our model and the function $h(\cdot)$ in \citet{hado2019generalbellman} serve distinct purposes. 
In our framework, $f(\cdot)$ is employed to convert logits into appropriate outputs, such as softmax scores. 
In contrast, $h(\cdot)$ in \citet{hado2019generalbellman} and the logarithm used in \citet{harm2019logmapping} primarily address issues related to the scale of rewards.

\section{MRP v.s. I.I.D.\ for Supervised Learning}
\label{sec:beliefcompare}

As our MRP formulation is fundamentally different from the traditional i.i.d.\ view, this section discusses the property of TD's solution and its potential benefits in the linear setting and beyond. 
The section concludes with a discussion on the merits of adopting an MRP versus an i.i.d.\ perspective.

\subsection{Connections between TD and OLS}
\label{sec:min_norm_equivalence}

The following proposition characterizes the connections between TD and OLS in the linear setting. Proof is in \cref{sec:app-proofs}.
\begin{restatable}{proposition}{tdolsequiprop}[Connection between TD and OLS]
\label{prop:equiv}
When $\Dmat$ has full support and $\Xmat$ has linearly independent rows, TD and OLS have the same minimum norm solution. 
Moreover, any solution to the linear system $\Xmat \wvec = \yvec$ must also be a solution to TD's linear system $\Amat \wvec = \bvec$ as defined in \cref{eq:A_b_def}.
\end{restatable}

\textbf{Empirical verification.} 
Table~\ref{tab:algo_relations} illustrates the distance between the closed-form solutions of our linear TD and OLS under various choices of the transition matrix by using a synthetic dataset (details are in Appendix \ref{sec:app-min-norm}). For \texttt{Random}, each element of $\Pmat$ is drawn from the uniform distribution $U(0, 1)$ and then normalized so that each row sums to one. 
The \texttt{Deficient} variant is exact the same as \texttt{Random}, except that the last column is set to all zeros before row normalization. 
This ensures that the last state is never visited from any state, thus not having full support in its stationary distribution. 
\texttt{Uniform} simply means every element of $\Pmat$ is set to $1/n$ where $n=100$ is the number of training points.
\texttt{Distance (Close)} assigns higher transition probability to points closer to the current point in the label space. 
Finally, \texttt{Distance (Far)} means contrary to \texttt{Distance (Close)}.
The last two variants are used to see if similarity between points in the label space can play a role in the transition when using our TD algorithm.

Two key observations emerge: 
1) As the feature dimension increases towards the overparameterization regime, both solutions become nearly indistinguishable, implying that designing $\Pmat$ may be straightforward when employing a powerful model like NN. 
2) Deficient choices for $\Pmat$ with non-full support can pose issues and should be avoided. 
In practice, one might opt for a computationally and memory-efficient $\Pmat$, 
such as a uniform constant matrix where every entry is set to $1/n$.
Such a matrix is ergodic and, therefore, not deficient. We will delve deeper into the selection of $\Pmat$ below. 

\begin{table*}[t]
\caption{ 
Distance between the closed-form min-norm solutions of TD and OLS $\|\vw_{TD}-\vw_{LS}\|_2$. 
Input matrix $\Xmat$ has normally distributed features with dimension $d\in \{70, 90, 110, 130\}$.
Results are average over $10$ runs with standard error in bracket.
Details can be found in Appendix \ref{sec:app-min-norm}.} 
\label{tab:algo_relations}
\setlength{\tabcolsep}{5pt}
\centering
\begin{tabular}{c|c|c|c|c}
\hline
\diagbox{Transition}{Dimension}
& 70 & 90 & 110 & 130 
\\ \hline
Random & 0.027 {\tiny ($\pm$ 0.01)} & 0.075 {\tiny ($\pm$ 0.02)} & $\le 10^{-10}$ {\tiny ($\pm$ 0.00)} & $\le 10^{-10}$ {\tiny ($\pm$ 0.00)}\\\hline
Uniform & 0.026 {\tiny ($\pm$ 0.01)} & 0.074 {\tiny ($\pm$ 0.01)} & $\le 10^{-10}$ {\tiny ($\pm$ 0.00)} & $\le 10^{-10}$ {\tiny ($\pm$ 0.00)}\\\hline
Distance (Far) & 0.028 {\tiny ($\pm$ 0.01)} & 0.075 {\tiny ($\pm$ 0.01)} & $\le 10^{-10}$ {\tiny ($\pm$ 0.00)} & $\le 10^{-10}$ {\tiny ($\pm$ 0.00)}\\\hline
Distance (Close) & 0.182 {\tiny ($\pm$ 0.10)} & 0.249 {\tiny ($\pm$ 0.04)} & $\le 10^{-10}$ {\tiny ($\pm$ 0.00)} & $\le 10^{-10}$ {\tiny ($\pm$ 0.00)}\\\hline
Deficient & 0.035 {\tiny ($\pm$ 0.01)} & 0.172 {\tiny ($\pm$ 0.04)} & 0.782 {\tiny ($\pm$ 0.18)} & 0.650 {\tiny ($\pm$ 0.25)}\\\hline
\end{tabular}
\end{table*}

\subsection{Statistical Efficiency and Variance Analysis}
\label{sec:variance_analysis}

A natural question is under what conditions the TD solution is better than the OLS solution, especially given that they may find the same solution under conditions specified above. 
Although the OLS estimator is known to be the best linear unbiased estimator (BLUE) under the i.i.d.\ assumption, our TD algorithm demonstrates the potential for a lower variance in settings with correlated noise, even when using a simple uniform transition matrix. 

Recall that conventional linear models assume $y_i = \vx_i^\top\vw^* + \epsilon_i$, where $\vw^*$ is the true parameters and $\epsilon_i$ is assumed to be independent noise. 
In contrast, under the MRP perspective, we consider the possibility of correlated noise. 
The following proposition shows when a TD estimator will be the most efficient one: 

\begin{proposition}
\label{prop:efficiency}
Suppose $\Amat$ as in \cref{eq:A_b_def} is invertible
and the error vector $\vepsilon$ satisfies $\EE[\vepsilon|\Xmat]=\zerovec$ and $\cov[\vepsilon|\Xmat]=\Cmat$. 
Then $\EE[\vw_{TD}|\Xmat]=\vw^*$ and the conditional covariance is
\begin{align}
\cov[\vw_{TD}|\Xmat]
=\Amat^{-1}\Xmat^\top\Smat
\Cmat
\Smat^\top\Xmat(\Amat^{-1})^{\top}
\end{align}
where $\Smat\defeq \Dmat(\Imat-\gamma \Pmat)$.
Moreover, if $\Smat=\Smat^\top$, the TD estimator is the BLUE for problems with $\Cmat=c\cdot\Smat^{-1},\forall c>0$. 
\end{proposition}
\begin{proof}
Recall from \cref{eq:A_b_def} that $\vw_{TD}=\Amat^{-1}\vb, \vb = \Xmat^\top \Dmat \vr$, and $\vr=(\Imat-\gamma \Pmat)\vy$. 
Therefore $\vw_{TD}-\vw^*$ equals
\begin{align}
\vw_{TD}-\vw^* &= \Amat^{-1} (\vb - \Amat \vw^*)= \Amat^{-1} (\Xmat^\top \Dmat \vr - \Amat \vw^*)\\
&=\Amat^{-1} (\Xmat^\top \Dmat (\Imat-\gamma \Pmat)\vy - \Amat \vw^*) \\
&= \Amat^{-1} (\Xmat^\top \Smat \vy - \Xmat^\top \Smat \Xmat \vw^*) =\Amat^{-1} \Xmat\Smat\vepsilon 
\end{align}
where we define $\Smat\defeq \Dmat(\Imat-\gamma \Pmat)$. 
When $\EE[\vepsilon|\Xmat]=\zerovec$, the conditional expectation of the above equation is zero.
Thus $\vw_{TD}$ is conditionally unbiased and its conditional covariance is
\begin{align}
\cov[\vw_{TD}|\Xmat]
&=\Amat^{-1}\Xmat^\top\Smat
\cdot\cov[\vepsilon|\Xmat]\cdot
\Smat^\top\Xmat(\Amat^{-1})^{\top}
\\
&=\Amat^{-1}\Xmat^\top\Smat
\cdot\Cmat\cdot
\Smat^\top\Xmat(\Amat^{-1})^{\top}
\end{align}
Finally, when $\Smat=\Smat^\top$ and $\Cmat=c\cdot\Smat^{-1}$ for some $c>0$, 
the TD estimator and its covariance become
\begin{align}
\vw_{TD}
=(\Xmat^{\top}\Smat\Xmat)^{-1}\Xmat\Smat\vy
\qquad \cov[\vw_{TD}|\Xmat]
=(c\Xmat^\top\Smat\Xmat)^{-1}.
\end{align}
By using the Cholesky decomposition $\Smat=\Lmat\Lmat^\top$, one can see that the TD estimator is equivalent to the OLS solution to a rescaled problem
$\widetilde{\vy}=\widetilde{\Xmat}\vw+\widetilde{\vepsilon}$ where
\begin{align}
\widetilde{\vy}
=\Lmat^\top\vy
\qquad \widetilde{\Xmat}
=\Lmat^\top\Xmat
\qquad \widetilde{\vepsilon}
=\Lmat^\top\vepsilon
\end{align}
Here $\cov[\widetilde{\vepsilon}|\Xmat]=\Lmat^\top\Cmat\Lmat=c\Imat$ so the TD estimator is the OLS solution to this problem and thus is the BLUE.
\end{proof}

\textbf{Remark 1}. 
This proposition identifies a situation in which our TD estimator outperforms other estimators, OLS included, in terms of efficiency. 
The condition $\Smat=\Smat^\top$ is needed so that there exists a corresponding symmetric covariance matrix $\Cmat$ for the data.
This condition implies that $\Dmat\Pmat=\Pmat^\top\Dmat$, or $D(s_i)P(s_j|s_i)=D(s_j)P(s_i|s_j)$, which means the Markov chain is reversible (i.e., detailed balance). 
Also note that $\Smat,\Amat$ are invertible under mild conditions (e.g., ergodic Markov chain). 
In such cases, the TD estimator also corresponds to the generalized least squares~(GLS) estimator~\citep{aitken1936ls,kariya2004generalized}. 
In practice, 
one may be tempted to directly estimate the covariance matrix as done by feasible GLS (FGLS)~\citep{badi2008econ}. 
However, estimating a covariance matrix is non-trivial as demonstrated in \cref{sec:experiment}. 
Furthermore, it should be emphasized that GLS/FGLS methods do not naturally support incremental learning, nor are they readily adaptable to deep learning models. 

\textbf{Remark 2.} The benefits of TD may not be limited to the situations described by the above proposition. 
It is challenging to mathematically describe these benefits because, under a more general $\Smat$, there is no intuitively interpretable form of the variance of TD's solution.

Furthermore, TD's benefits can be expanded by incorporating its more recent variants, such as emphatic TD~\citep{sutton2016etd} with transition-based $\gamma$ and interest weighting, which may offer more flexibility in designing the matrix $\Smat$. Other variants include gradient TD~\citep{maei2011gradient} and quasi-second order accelerated TD~\citep{pan2017atd,pan2017sketchatd}. 
We leave these additional studies for future work.

Below, we provide a more general perspective to understand the benefits of TD in terms of variance reduction.
The basic idea is that when the ground truth target variables of consecutive time steps are positively correlated, the TD target benefits from a reduction in variance.

\begin{proposition}
\label{prop:variance}
Assume the estimated next-state value $\widehat{y}_{t+1}$ satisfies $\widehat{y}_{t+1}=\EE[y_{t+1}]+\epsilon$ where $\epsilon$ is some independent noise with zero mean and standard deviation $\sigma_\epsilon$. 
Let $\sigma_t, \sigma_{t+1}$ be the standard deviations of $y_t, y_{t+1}$ respectively, and $\rho_{t,t+1}$ be the Pearson correlation coefficient between $y_{t}$ and $y_{t+1}$. 
If $\rho_{t,t+1}\ge \frac{\gamma^2 (\sigma_{t+1}^2 + \sigma_\epsilon^2)}{2\gamma \sigma_t \sigma_{t+1}}$, then $\var(y_{t,td}) \le \var(y_t)$.
\end{proposition}
\begin{proof}
We rewrite the TD target (\ref{eq:td-target}) as
\begin{align}
y_{t,td}
&=(y_{t}-\gamma y_{t+1})+\gamma\EE[y_{t+1}]+\gamma \epsilon
=y_{t}-\gamma(y_{t+1} - \EE[y_{t+1}])+\gamma \epsilon
\end{align}
This means we can treat $y_{t+1}$ as a control variate and the variance of this estimate is
\begin{align}
\var(y_{t,td})
&=
\var(y_{t})
+\gamma^2 \var(y_{t+1})
-2\gamma \cov(y_t, y_{t+1})
+\gamma^2 \var(\epsilon)
\\
&=
\sigma_{t}^2
+\gamma^2 \sigma_{t+1}^2
-2\gamma \rho_{t,t+1}\sigma_{t}\sigma_{t+1}
+\gamma^2 \sigma_\epsilon^2
\end{align}
Plugging into the condition on $\rho_{t,t+1}$ would yield $\var(y_{t,td}) \le \var(y_t)$.
\end{proof}

\textbf{Remark.} To better interpret the result, we can consider $\sigma_t=\sigma_{t+1}$, then the variance is simplified to
\begin{align}
\var(y_{t,td})
=
\sigma_{t}^2
+\gamma^2 \sigma_{t}^2
-2\gamma \rho_{t,t+1}\sigma_{t}^2
+\gamma^2 \sigma_\epsilon^2,
\end{align}
which is a convex function in $\gamma$ and achieves its lowest value when 
\begin{align}
\gamma=\frac{\rho_{t,t+1}\sigma_t^2}{\sigma_t^2 + \sigma_\epsilon^2} \propto \rho_{t,t+1}.
\end{align}
This suggests that the stronger the correlation, the more (i.e., a larger $\gamma$) we might rely on the bootstrap term to reduce variance, coinciding with our intuition. 

\textbf{Empirical verification}.
Here we verify that when the outputs are indeed positively correlated, our method can generalize better than OLS.  
To this end, we run experiments using a Gaussian process with positive correlated outputs. 

In each run, we jointly sample $n=200$ data points (100 for training and 100 for test) from a Gaussian process where each element of the input matrix $\Xmat$ is drawn from the standard normal distribution $\Ncal(0,1)$. 
The input dimension is $d=70$.
For the outputs, the mean function is given by $m(\vx)=\vx^\top\vw^*$ where $\vw^*$ is a vector of all ones and the covariance matrix is a block diagonal matrix 
\begin{equation}
\Cmat=
\begin{bmatrix}
\widetilde{\Cmat}\\
&\ddots\\
&&\widetilde{\Cmat}
\end{bmatrix}_{200\times 200}
\quad\text{ with }\quad
\widetilde{\Cmat}=
\begin{bmatrix}
1 & \rho & \dots & \rho \\
\rho & 1 & \dots & \rho \\
\vdots & \vdots & \ddots & \vdots \\
\rho & \rho & \dots & 1 \\
\end{bmatrix}_{10\times 10}
\end{equation}
where $\rho$ is a tuning parameter for the correlation. 
When $\rho>0$, this structure ensures that all ten points within the same ``cluster'' are positively correlated.
Finally, we add an independent noise (with zero mean and standard deviation of 0.1) to each output before using them for training and testing. 

For our TD method, we set $\gamma=0.99$ and design the probability matrix as an interpolation between a covariance matrix $\Cmat$ and $\onevec-\Cmat$, defined as \( \Pmat = (1-\eta)(\onevec-\Cmat) + \eta \Cmat \), followed by normalization to ensure that it forms a valid stochastic matrix. We vary \(\eta\) over the set \{0.5, 0.6, 0.7, 0.8, 0.9\} and the correlation coefficient \(\rho\) over \{0.1, 0.3, 0.5, 0.7, 0.9\}.
As $\eta$ gets closer to one, the transition matrix agrees more with the covariance matrix so that it is more likely to transition to positively correlated points. 
With 100 training points, we learn both the TD solution $\vw_{TD}$ and OLS solution $\vw_{LS}$ and plot their test RMSE.
The experiment is repeated 50 times and \cref{fig:gp_P} reports the mean and standard deviation for different $(\rho,\eta)$ values.

\begin{figure}
  \begin{center}
    \includegraphics[width=0.7\textwidth]{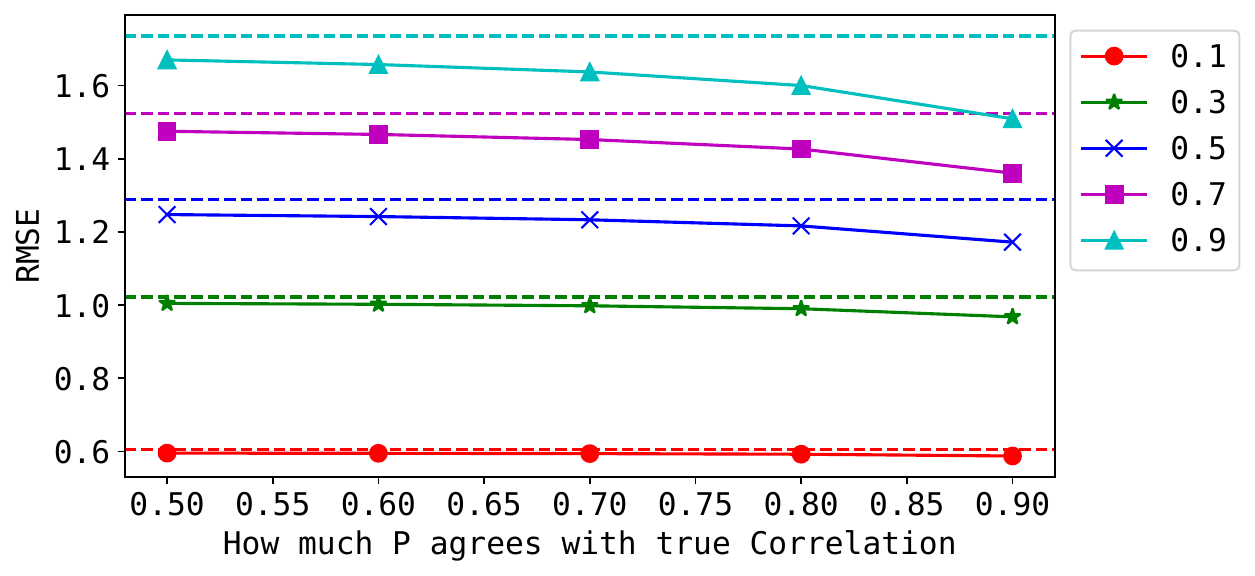}
  \end{center}
  \caption{Comparing TD (solid) with OLS (dashed) with different settings.}
  \label{fig:gp_P}
\end{figure}

The results are shown in \cref{fig:gp_P}.
Under different levels of correlation, as indicated by various colors, TD (solid lines) increasingly outperforms OLS (dashed lines) when the transition probability matrix $\Pmat$ more closely aligns with the covariance. 
$\Pmat$ is designed so that a higher transition probability is assigned between two points when their outputs exhibit a stronger positive correlation.
Although there is no direct theoretical result explicitly proving the degradation of OLS performance under increasingly correlated data, existing works, such as Lemma 3 in \citet{song2023hybrid} and Proposition 4 in \citet{kishan2024offonrl}, indirectly suggest that its generalization error may be higher than in the i.i.d.\ setting. 
This is because their upper bounds on the error include a linear dependence on either the number of steps or the sample size, whereas in the i.i.d.\ case, the dependence is typically less costly.

\textbf{Practical choices: MRP view v.s.\ i.i.d.\ view.} 
We conclude this section by discussing the practical considerations on either using the MRP or i.i.d.\ view. 
In machine learning, the selection of algorithms often depends on whether their underlying assumptions align with the ground truth of the data -- a factor that is fundamentally unknowable. 
Thus, the decision often boils down to the user's belief -- whether to treat the data as MRP or i.i.d. 
It is noteworthy that the TD algorithm demonstrates greater generality as it: 1) accommodates the solution of the linear system as demonstrated in \cref{prop:equiv} and is equivalent to OLS under certain conditions; 2) offers a straightforward approach to setting the discount factor to zero, effectively reducing the bootstrap target to the original SL training target.

The theoretical and empirical results from this section suggest: 1) in the absence of prior knowledge, adopting TD with small $\gamma$ values and a uniform $\Pmat$ could be beneficial for computational and memory efficiency since we can simply pick the next data point uniformly at random without storing specific probabilities; if the ground truth suggests that the data points are positively correlated, this approach might yield a performance gain; 2) when it is known that two points are positively correlated, one could enhance variance reduction by strategically encouraging transition from one point to another.

\section{Convergence Analysis}
\label{sec:theory}
In this section, we present convergence results for our generalized TD algorithm (Algorithm~\ref{alg_gltd}) under both the expected updating rule and the sample-based updating rule. Detailed proofs are provided in \cref{sec:app-proofs}.

We show that the expected updating rule of our generalized TD algorithm can be viewed as a generalized Bellman operator (\ref{eq:expected_update}) that is contractive. 
Based on this, we further prove finite-time convergence when applying TD(0) updates with linear function approximation.
We primarily follow the convergence framework presented in \citet{bhandari2018finite}, making non-trivial adaptations due to the presence of the inverse link function.
We start with the assumptions.

\begin{assumption}[Feature regularity]
\label{assume:feature_regularity}
$\forall s\in\Scal,\|\vx(s)\|_2\le 1$ and
the steady-state covariance matrix $\Sigmamat\defeq D(s)\vx(s)\vx(s)^\top$ has full rank.
\end{assumption}

\cref{assume:feature_regularity} is a typical assumption necessary for the existence of the fixed point when there is no transform function~\citep{tsitsiklis1997analysis}. 

\begin{assumption}
\label{assume:link_function}
The inverse link function $f:\RR\mapsto\Ycal$ is continuous, invertible and strictly increasing. 
Moreover, it has a bounded derivative $f'$ on any bounded domain.
\end{assumption}

\textbf{Remark.} 
\cref{assume:link_function} is satisfied for those inverse link functions commonly used in GLMs, including but not limited to identity function (linear regression), exponential function (Poisson regression), and sigmoid function (logistic regression)~\citep[Table~2.1]{mccullagh1989generalized}.

For theoretical analysis, we consider training in a convex and compact set $\Wcal\subset\RR^d$ such that the parameters are projected back to it after each update (see Equations (\ref{eq:expected_update}) and (\ref{eq:sample_update})). 
Without this, one may need to consider certain special cases where the parameters $\vw$ may diverge to infinity, akin to the case of (unconstrained) logistic regression.
With this in mind, the following lemma holds.

\begin{restatable}{lemma}{lipschitzlemma}
\label{lemma:lipschitz}
Under \cref{assume:feature_regularity}, \ref{assume:link_function}, for $\vw\in\Wcal$, there exists $L\ge 1$ such that
$\forall s_1,s_2\in\Scal$ with $z_1=\vx(s_1)^\top\vw,z_2=\vx(s_2)^\top\vw$,
\setlength{\abovedisplayskip}{3pt}
\setlength{\belowdisplayskip}{3pt}
\begin{small}
\begin{align}
\frac{1}{L}|z_1-z_2|\le|f(z_1)-f(z_2)|\le L|z_1-z_2|.
\label{eq:bi_lipschitz}
\end{align}
\end{small}
\end{restatable}

This lemma is important for relating the bounds between the logit space and the output/label space.
Let $z(s)=\vx(s)^\top\vw$, or $z=\vx^\top\vw$ for conciseness in the following.
The next assumption is necessary later to ensure that the step size is positive.

\begin{assumption}[Bounded discount]
\label{assume:bound_discount}
The discount factor satisfies $\gamma<\frac{1}{L^2}$ for the $L$ in \cref{lemma:lipschitz}.
\end{assumption}

Given these assumptions, \cref{sec:expected_contraction} shows that the projected expected update admits a unique fixed point $\vw^*\in\Wcal$, \cref{sec:expected_convergence} proves finite-time convergence bounds for the projected expected update, and finally, \cref{sec:sample_convergence} proves the convergence bound of sample-based update.

\subsection{Contraction of Expected Update}
\label{sec:expected_contraction}

The expected update followed by a projection onto $\Wcal$ is given by
\begin{align}
\vw_{t+1} &= 
\Pcal(\vw_t + \alpha \overline{g}(\vw_t))
\label{eq:expected_update}
\end{align}
where $\overline{g}(\vw_t)\defeq \EE[(y_{td}-y)\vx]$ is the expected update,
$\Pcal(\vw) \defeq \argmin_{\vu\in\Wcal}\|\vw-\vu\|_2^2$ is the projection,
$y\defeq f(z),y_{td}\defeq f(r+\gamma \vx'^\top\vw_t)$ and $\vx'\defeq \vx(s')$. 
This updating rule -- the mapping from $\vw_t$ to $\vw_{t+1}$ -- can be viewed as a type of \emph{generalized Bellman operator}. 
The term ``generalized" refers to its ability to incorporate different types of $f$ functions, enabling it to accommodate various learning tasks, such as regression and classification. 

We show the existence of a fixed point $\vw^*\in\Wcal$ for our update using the Banach fixed-point theorem. 
Given two iterates $\vw_t^A,\vw_t^B\in\Wcal$ and their updates
\begin{align}
\vw_{t+1}^A
=\Pcal(\vw^A_t+\alpha \overline{g}(\vw^A_t))
\qquad
\vw_{t+1}^B
=\Pcal(\vw^B_t+\alpha \overline{g}(\vw^B_t))
\end{align}
we will show that any two parameters $\vw^A,\vw^B\in\Wcal$ will get closer after each iteration.

Note that
\begin{align}
\|\vw^A_{t+1}-\vw^B_{t+1}\|_2^2
&\le
\|\vw^A_{t}+\alpha \overline{g}(\vw^A_t)-(\vw^B_{t}+\alpha \overline{g}(\vw^B_t))\|_2^2
\end{align}
since projection $\mathcal{P}$ onto a convex set is contracting.
Furthermore, we can expand it as
\begin{align}
\|\vw^A_{t+1}-\vw^B_{t+1}\|_2^2
&\le
\|\vw^A_t-\vw_{t}^B\|_2^2
-2\alpha(\vw_t^A - \vw_{t}^B)^\top (\overline{g}(\vw_t^B)-\overline{g}(\vw_t^A))
\nonumber\\
&\qquad\qquad 
+\alpha^2\|\overline{g}(\vw_t^A)-\overline{g}(\vw_t^B)\|_2^2
\end{align}
The common strategy is to ensure that the second term can outweigh the third term so that the distance decreases after each iteration. 
To simplify notation, denote their predictions $z^A=\vx^\top \vw^A,z^B=\vx^\top \vw^B$, transformed predictions $y^A=f(z^A),y^B=f(z^B)$ and bootstrap targets $y_{td}^A=f(r+\vx'^\top\vw^A),y_{td}^B=f(r+\vx'^\top\vw^B)$.
The following two lemmas bound the two terms respectively.

\begin{restatable}{lemma}{contractioncrossprod}
\label{lemma:cross_prod}
For $\vw^A,\vw^B\in\Wcal$,
$(\vw^A-\vw^B)^\top (\overline{g}(\vw^B)-\overline{g}(\vw^A))
\ge \left(\frac{1}{L}-\gamma L\right)\cdot 
\EE\left[(z^A-z^B)^2\right]$.
\end{restatable}

\begin{restatable}{lemma}{contractiongradnorm}
\label{lemma:grad_norm}
For $\vw^A,\vw^B\in\Wcal$,
$\|\overline{g}(\vw^A)-\overline{g}(\vw^B)\|_2
\le 2L \sqrt{\EE\left[(z^A-z^B)^2\right]}$
\end{restatable}

With these lemmas, we can prove the following contraction theorem.

\begin{restatable}{theorem}{expectcontraction}[Contraction of Expected Update]
\label{thm:expected_contraction}
Under \cref{assume:feature_regularity}-\ref{assume:bound_discount},
by choosing $\alpha=\frac{1-\gamma L^2}{4L^3}>0$, 
the projected update (\ref{eq:expected_update}) is a contraction w.r.t.\ $\|\cdot\|_2$ with a unique fixed point $\vw^*\in\Wcal$, and $\vw_t$ converges to it.
\end{restatable}

With this well-defined fixed point, we will analyze the finite-time convergence rate next.

\subsection{Convergence Rate of Expected Update}
\label{sec:expected_convergence}

Now we prove the convergence rate.
The distance from $\vw_{t+1}$ to $\vw^*$ can be expanded as
\begin{align}
\|\vw^* - \vw_{t+1}\|_2^2
=
\|\vw^* - \vw_{t}\|_2^2
- 2\alpha(\vw^* - \vw_{t})^\top \overline{g}(\vw_t)
+\alpha^2 \|\overline{g}(\vw_t)\|_2^2
\end{align}
Similarly to before, we want to ensure that
the second term can outweigh the third term in the RHS so that $\vw_{t+1}$ can get closer to $\vw^*$ than $\vw_{t}$ in each iteration. 
This is achieved again by applying \cref{lemma:cross_prod} and \cref{lemma:grad_norm}, which leads to the following result:

\begin{restatable}{theorem}{expectconverge}[Convergence Rate of Expected Update]
\label{thm:expected_update_rate}
Under \cref{assume:feature_regularity}-\ref{assume:bound_discount},
consider the sequence $(\vw_0,\vw_1,\cdots)$ satisfying 
\cref{eq:expected_update}.
Let $\overline{\vw}_T\defeq \frac{1}{T}\sum_{t=0}^{T-1}\vw_t$, 
$\overline{z}_T=\vx^\top\overline{\vw}_T$
and $z^*=\vx^\top\vw^*$. 
By choosing $\alpha=\frac{1-\gamma L^2}{4L^3}>0$, we have
\begin{align}
\EE\left[ (z^* - \overline{z}_T)^2 \right]
&\le
\left(\frac{2L^2}{1-\gamma L^2}\right)^2
\frac{\|\vw^*-\vw_{0}\|_2^2}{T}
\label{eq:expected_z_converge}
\\
\|\vw^*-\vw_{T}\|_2^2
&\le 
\exp\left(-T\omega
\left(\frac{1-\gamma L^2}{2L^2}\right)^2 \right)
\|\vw^*-\vw_{0}\|_2^2
\label{eq:exp_converge}
\end{align}
\end{restatable}

\cref{eq:expected_z_converge} shows that in expectation, the average prediction converges to the true value in the $z$-space, while \cref{eq:exp_converge} shows that the last iterate converges to the fixed point exponentially fast when using expected update.
In practice, a sample-based update is preferred and we discuss its convergence next.

\subsection{Convergence Rate of Sample-based Update} 
\label{sec:sample_convergence}

Here we show the convergence under i.i.d.\ sample setting. 
Suppose $s_t$ is sampled from the stationary distribution $D(s)$ and $s_{t+1}\sim P(\cdot|s_{t})$.
For conciseness, let $\vx_t\defeq \vx(s_{t})$ 
and $\vx_{t+1}\defeq \vx(s_{t+1})$.
The sample-based updating rule is
\begin{align}
\vw_{t+1}= 
\Pcal(\vw_t + \alpha_t g_t(\vw_t))
\quad
\text{ with }
\quad
g_t(\vw_t)
\defeq (y_{t,td}-y_t)\vx_t
\label{eq:sample_update}
\end{align}
where $y_t\defeq f(\vx_t^\top\vw_t),y_{t,td} \defeq f(z_{t,td})=f(r_{t+1}+\gamma\vx_{t+1}^\top\vw_t)$. 
The $t$ in $g_t$ is for the sample $(\vx_t,y_t)$, while the $\vw_t$ in $g_t(\vw_t)$ is for calculating $y_t$ and $y_{t,td}$.

To account for randomness, let $\sigma^2\defeq \EE[\|g_{t}(\vw^*)\|_2^2]$, the variance of the TD update at the stationary point $\vw^*$ under the stationary distribution. 
We need two supporting lemmas to prove the convergence. 
The first shows that there is an optimality condition for $\vw^*$ similar to the first-order condition in constrained convex optimization (\cref{lemma:optimality_condition}), and the second bounds the expected norm of the update (\cref{lemma:grad_norm_sample}). 

For the first lemma, note that the iterative update (\ref{eq:expected_update}) does not apply gradient descent on any fixed objective, so one cannot simply apply the first-order optimality condition from constrained optimization~\citep[Sec.4.2.3]{boyd2004convex}.
Nevertheless, we can still prove the condition due to the projection step:
\begin{restatable}{lemma}{optimalitycondition}
\label{lemma:optimality_condition}
The fixed point $\vw^*\in\Wcal$ in \cref{thm:expected_contraction} satisfies 
\begin{align}
(\vw^*-\vw)^\top \overline{g}(\vw^*)
\ge 0
\quad \forall \vw\in\Wcal.
\label{eq:optimality_condition}
\end{align}
\end{restatable}

The next lemma is similar to \cref{lemma:grad_norm} but for the sample-based update:

\begin{restatable}{lemma}{gradnormsample}
\label{lemma:grad_norm_sample}
For $\vw\in\Wcal$, 
$\EE[\|g_t(\vw)\|_2^2]\le 2\sigma^2+8L^2
\EE[(z_t^*-z_t)^2]$ where $\sigma^2=\EE[\|g_{t}(\vw^*)\|_2^2]$.
\end{restatable}

Now we are ready to present the convergence when using i.i.d.\ sample for the update:

\begin{restatable}{theorem}{sampleconverge}[Convergence Rate of Sampled-based Update]
\label{thm:sample_converge_rate}
Under \cref{assume:feature_regularity}-\ref{assume:bound_discount},
with sample-based update \cref{eq:sample_update}, 
let 
$\sigma^2=\EE[\|g_{t}(\vw^*)\|_2^2]$,
$\overline{\vw}_T\defeq \frac{1}{T}\sum_{t=0}^{T-1}\vw_t$, 
$\overline{z}_T=\vx^\top\overline{\vw}_T$
and $z^*=\vx^\top\vw^*$.
For $T\ge \frac{64 L^6}{(1-\gamma L^2)^2}$ and a constant step size $\alpha_t=1/\sqrt{T},\forall t$, we have
\begin{align}
\EE\left[ (z^* - \overline{z}_T)^2 \right]
\le 
\frac{L\left(
\|\vw^*-\vw_0\|_2^2 + 2\sigma^2
\right)}
{\sqrt{T}(1-\gamma L^2)}.
\end{align}
\end{restatable}

This shows that the generalized TD update converges even when using the sample-based update.
Not surprisingly, it is slower than using the expected update (\ref{eq:expected_z_converge}).

\section{Empirical Studies}
\label{sec:experiment}

This section has two primary objectives: first, to validate whether the theoretical results of variance analysis are applicable to real-world datasets in both linear and neural network (NN) settings; second, to investigate the practical utility of our algorithm within a standard learning setting. To accomplish these goals, we have structured the following content into four subsections, respectively:
1) In the linear setting, we examine synthetic noise added to a real-world dataset.
2) In a NN setting, we explore the impact of synthetic noise on a real-world dataset.
3) In a NN setting, we assess our algorithm's performance using a dataset where the training targets may intuitively be correlated.
4) We evaluate our algorithm on commonly used real-world datasets from regression to image classification.

Our conclusions are as follows: 1) TD learning performs in line with our theoretical expectations in either linear or neural network settings with correlated noise; 2) in a standard SL setting, TD demonstrates performance on par with traditional SL methods. Results on additional data and any missing details are in~\cref{sec:app-reproduce-real}.

\textbf{Setup overview.} 
For regression with real targets, we adopt the popular dataset execution time~\citep{exectime}. 
For image classification, we use the popular MNIST~\citep{mnist}, Fashion-MNIST~\citep{mnistfashion}, Cifar10 and Cifar100~\citep{cifar} datasets. Additional results on other popular datasets (e.g. house price, bike rental, weather, insurance) are in ~\cref{sec:app-reproduce-real}.
In TD algorithms, unless otherwise specified, we use a fixed transition probability matrix with all entries set to $1/n$, which is memory and computation efficient, and simplifies sampling processes.  

\textbf{Baselines and naming rules.} 
TDReg: our TD approach, with its direct competitor being Reg (conventional $l_2$ regression). 
Reg-WP: Utilizes the same probability transition matrix as TDReg but does not employ bootstrap targets.
This baseline can be used to assess the effect of bootstrap and transition probability matrix. 

\subsection{Linear Setting with Correlated Noise} 

As outlined in~\cref{sec:variance_analysis}, we anticipate that in scenarios where the target noise exhibits positive correlation, TD learning should leverage the advantages of using a transition matrix. When this matrix transitions between points with correlated noise, the TD target counterbalances the noise, thereby reducing variance. 
As depicted in~\cref{fig:house-rho-P}, we observe that with an increasing correlation coefficient (indicative of progressively positively correlated noise), TD consistently shows improved generalization performance towards the underlying best baseline. 

In these figures, alongside TD and OLS, we include two baselines known for their efficacy in handling correlated noise: generalized least squares (GLS) and feasible GLS (FGLS)~\citep{aitken1936ls}. It is noteworthy that GLS operates under the assumption that the covariance of noise generation is fully known, aiming to approximate the underlying optimal estimator $\wvec^*$. FGLS has two procedures: estimate the covariance matrix followed by applying this estimation to resolve the linear system. Please refer to \cref{sec:app-additional-realworldlinear} for details.

\begin{figure*}
  \centering
  \includegraphics[width=\textwidth]{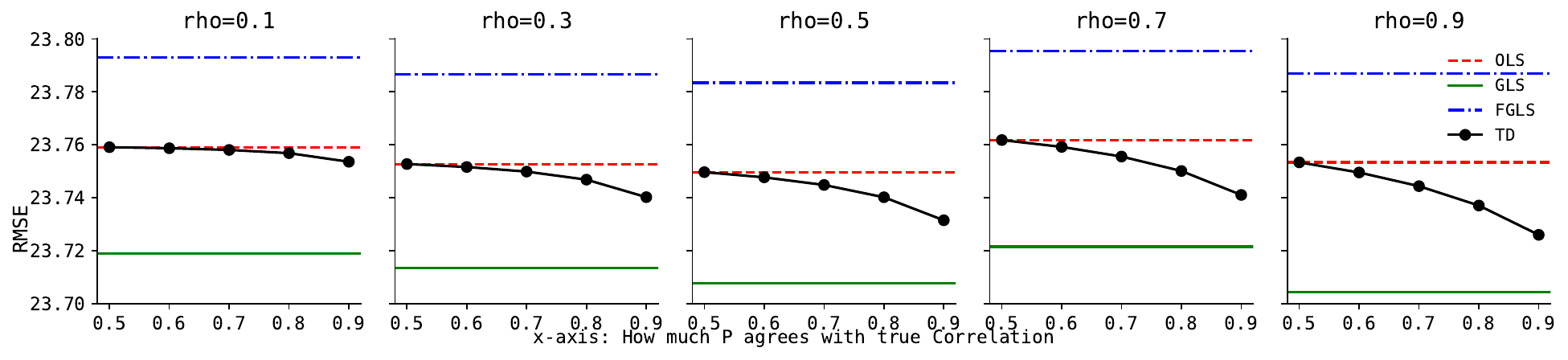}
  \caption{Test root mean squared error (RMSE) versus different \( Ps \) on exec time dataset. From left to right, the noise correlation coefficient increases (\( \rho \in \{0.1, 0.3, 0.5, 0.7, 0.9\} \)). Each plot's x-axis represents the degree of alignment between the transition matrix \( P \) and the true covariance matrix that generates the noise. As \( P \) aligns better – implying a higher likelihood of data points with positively correlated noise transitioning from one to another – the solution of TD increasingly approximates the $\wvec^*$. Furthermore, as the correlation among the data intensifies, one can see larger gap between TD and OLS/FGLS.}
  \label{fig:house-rho-P}
\end{figure*}

\subsection{Deep learning: Regression with Synthetic Correlated Noise}

Similar to the previous experiment, we introduced noise to the original training target using a GP and opted for a uniform transition matrix. The results, depicted in~\cref{fig:extimelc}, reveal two key observations: 1) As the noise level increases, the performance advantage of TD over the baselines becomes more pronounced; 2) The baseline Reg-WP performs just as poorly as conventional Reg, underscoring the pivotal role of the TD target in achieving performance gains. Additionally, it was observed that all algorithms select the same optimal learning rate even when smaller learning rates are available, and TD consistently chooses a large $\gamma=0.95$. This implies that TD's stable performance is attributed NOT to the choice of a smaller learning rate, but rather to the bootstrap target. Additional results are in~\cref{fig:houselc}.

\begin{figure}
  \centering
 \includegraphics[width=0.8\textwidth]{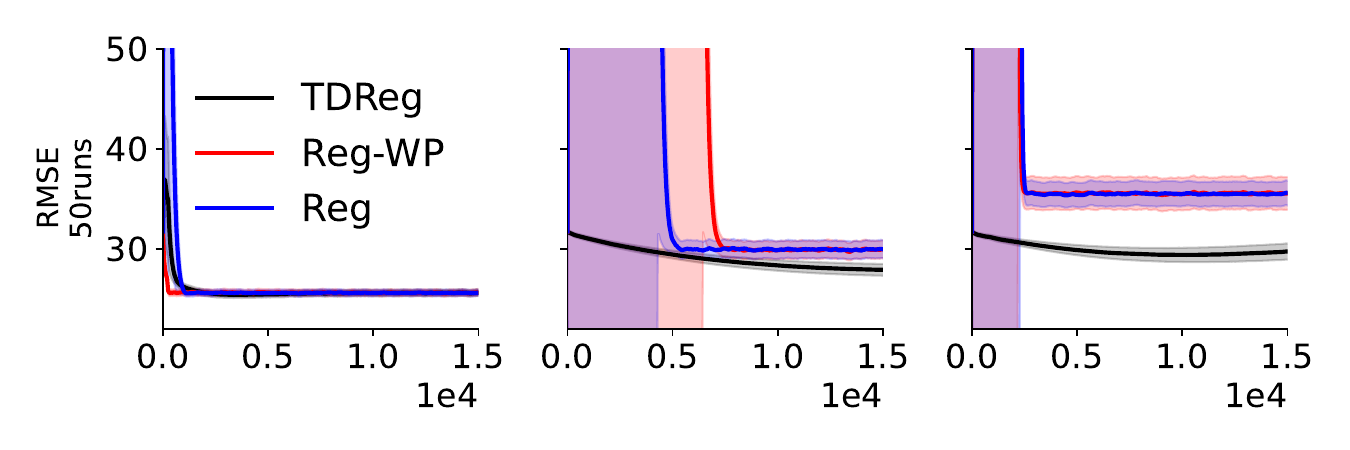}
  \caption{
  Learning curves on execute time dataset. From left to right, the noise defined as $c\times \epsilon$ where $\epsilon$ is noise generated by GP process and $c\in \{10, 20, 30\}$. Due to the addition of large-scale noise, Reg-WP and Reg exhibit high instability during the early learning stages and ultimately converge to a suboptimal solution compared to our TD (in black line) approach.
  }
  \label{fig:extimelc}
\end{figure}

\subsection{Deep Learning without Synthetic Noise}

In order to test the practical utility of our algorithm, we evaluate it using various transition matrices on an air quality dataset~\citep{vito2016airquality}, without introducing any synthetic noise to alter the original data. This dataset's task is to predict CO concentration. The main objective here is to demonstrate that real-world problems might also feature positively correlated data, where our method could outperform SGD.

Although it is impossible to determine the underlying ground truth of the noise structure precisely, we propose the following intuition for potential positive correlations among data points in this dataset.

If air quality measurements are taken at regular intervals (e.g., hourly or daily) at the same location, temporal dependencies might arise between consecutive measurements due to factors such as diurnal variations, weather patterns, or pollution sources. Consequently, the CO concentration measured at one time point may be positively correlated with that measured at adjacent times. Additionally, spatial dependencies could occur if measurements are taken at nearby locations, particularly in densely populated urban areas with unevenly distributed pollution sources. In such cases, neighboring locations may experience similar air quality conditions, resulting in a positive correlation between their CO concentrations.

\begin{figure}
  \centering
  \includegraphics[width=0.8\textwidth]{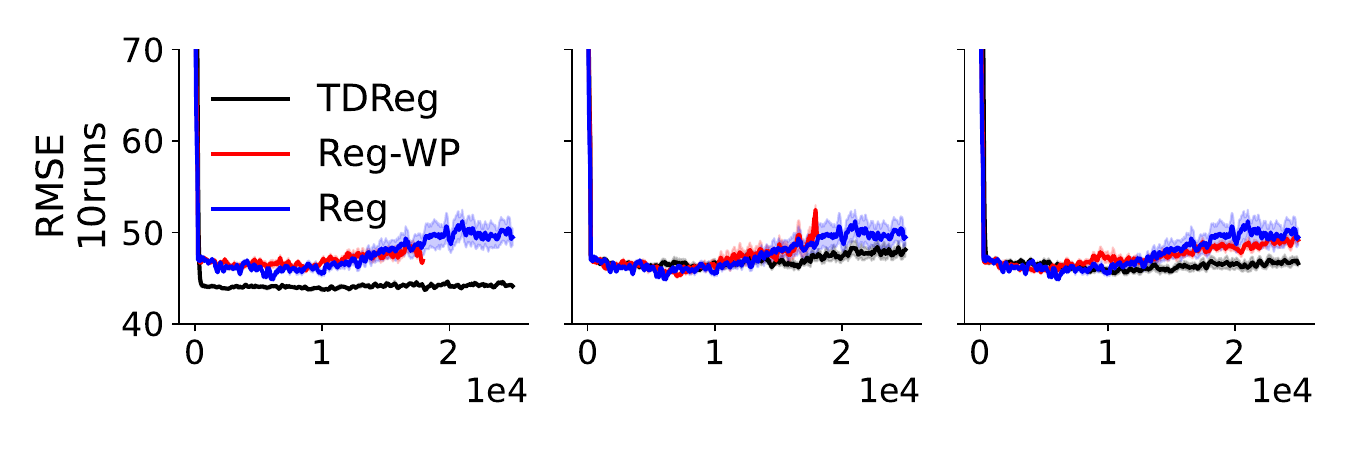}
  \caption{We show the test RMSE with error vs. the number of training steps when using three types of transition matrices. From left to right, the first type of transition assigns a high probability of transitioning to similar data points, the second assigns a high probability of transitioning to distant points, and the third one uses an uniform constant transition matrix as described before. The results are averaged over $10$ random seeds.}
  \label{fig:airquality}
\end{figure}

Figure~\ref{fig:airquality} shows the learning curves of TD and baselines using different types of transition matrix $P$. We believe the experiments shed significant light on our approach, based on the following observations: 

1. When dealing with potentially correlated data, the conventional SGD algorithm (Reg) may diverge or settle on a suboptimal solution, as it theoretically requires i.i.d.\ data to converge. In our experiments, SGD did not perform well. 

2. Our algorithm (TDReg) demonstrates optimal performance when using a transition matrix $P$ that favors transitions between closer points. It performs less well when this preference is reversed and shows only a modest improvement over SGD with a uniform matrix. This is possibly because the entire dataset is somewhat positively correlated, and the use of a bootstrap target helps mitigate the effect of noise. 

3. The comparison with baseline Reg-WP, which differs from Reg only in its use of the same transition matrix as TDReg, aims to isolate the effect of the bootstrap target. If Reg-WP performed as well as TDReg, it would suggest that the benefits of TDReg stem solely from the transition matrix rather than the bootstrap mechanism. However, the experiments show that without the bootstrap, Reg-WP may diverge (as indicated by the disappearance of the red line in the figures), likely because the transition matrix $P$ further increases data point correlation, causing the conventional algorithm to fail. 

As an additional note, to exclude the possibility of divergence due to excessively high learning rates in the baseline algorithms, we tested with sufficiently small learning rates. Despite this, the baselines did not perform better, indicating that the divergence was not due to the selection of a larger learning rate of TD.

\subsection{Deep Learning: Standard Problems}
\label{sec:experiments-standard}

This section aims at investigating how our TD algorithm works on commonly used datasets where there may not be correlation among targets. We found that there is no clear gain/lose with TD's bootstrap target in either regression~(\cref{tab:deep_reg}) or image classification tasks~(\cref{tab:image}). For image datasets, we employed a convolutional neural network (CNN) architecture to demonstrate TD's practical utility. 
\begin{table}[t]
\caption{Test root mean squared error (RMSE) with standard error. The results have been smoothed using a 5-point window before averaging over 5 runs.} 
\label{tab:deep_reg}
\setlength{\tabcolsep}{5pt}
\centering
\begin{tabular}{c|c|c|c}
\hline
\diagbox{Data}{Alg.}
& \textbf{TD-Reg} & \textbf{Reg-WP} & \textbf{Reg} \\
\hline
house & 3.384 {\tiny $\pm$ 0.21} & 3.355 {\tiny $\pm$ 0.23} & 3.319 {\tiny $\pm$ 0.16} \\
\hline
exectime & 23.763 {\tiny $\pm$ 0.38} & 23.794 {\tiny $\pm$ 0.36} & 23.87 {\tiny $\pm$ 0.36} \\
\hline
bikeshare & 40.656 {\tiny $\pm$ 0.77} & 40.904 {\tiny $\pm$ 0.36} & 40.497 {\tiny $\pm$ 0.45} \\
\hline
\end{tabular}
\end{table}

\begin{table}[htbp]
\caption{
Image classification test accuracy. The results are smoothed over $10$ evaluations before averaging over $3$ random seeds. The standard errors are small with no definitive advantage observed for either algorithm.}
\label{tab:image}
\centering
\begin{tabular}{c|c|c}
\hline
\diagbox{Dataset}{Algs} & \textbf{TD-Classify} & \textbf{Classify} \\
\hline
\textbf{mnist} & $99.06\%$ & $99.00\%$  \\
\hline
\textbf{mnistfashion} & $89.10\%$ & $88.96\%$ \\
\hline
\textbf{cifar10} & $67.42\%$ & $67.13\%$ \\
\hline
\textbf{cifar100} & $31.55\%$ & $31.21\%$ \\
\hline
\end{tabular}
\end{table}

We then further study the hyperparameter sensitivity in ~\cref{fig:sensi}. With NNs, TD can pose challenges due to the interplay of two additional hyperparameters: $\gamma$ and the target NN moving rate $\tau$~\citep{mnih2015humanlevel}. Optimizing these hyperparameters can often be computationally expensive. Therefore, we investigate their impact on performance by varying $\gamma, \tau$. We find that selecting appropriate parameters tends to be relatively straightforward. \cref{fig:sensi} displays the testing performance across various parameter settings. It is evident that performance does not vary significantly across settings, indicating that only one hyperparameter or a very small range of hyperparameters needs to be considered during the hyperparameter sweep. It should be noted that~\cref{fig:sensi} also illustrates the sensitivity analysis when employing three intuitive types of transition matrices that do not require prior knowledge. It appears that there is no clear gain/lose with these choices in a standard task. We refer readers to Appendix~\ref{sec:app-designp} for details. 

\begin{figure}
  \centering
  \subfigure[Bikeshare]{
    \includegraphics[width=\figwidththree]{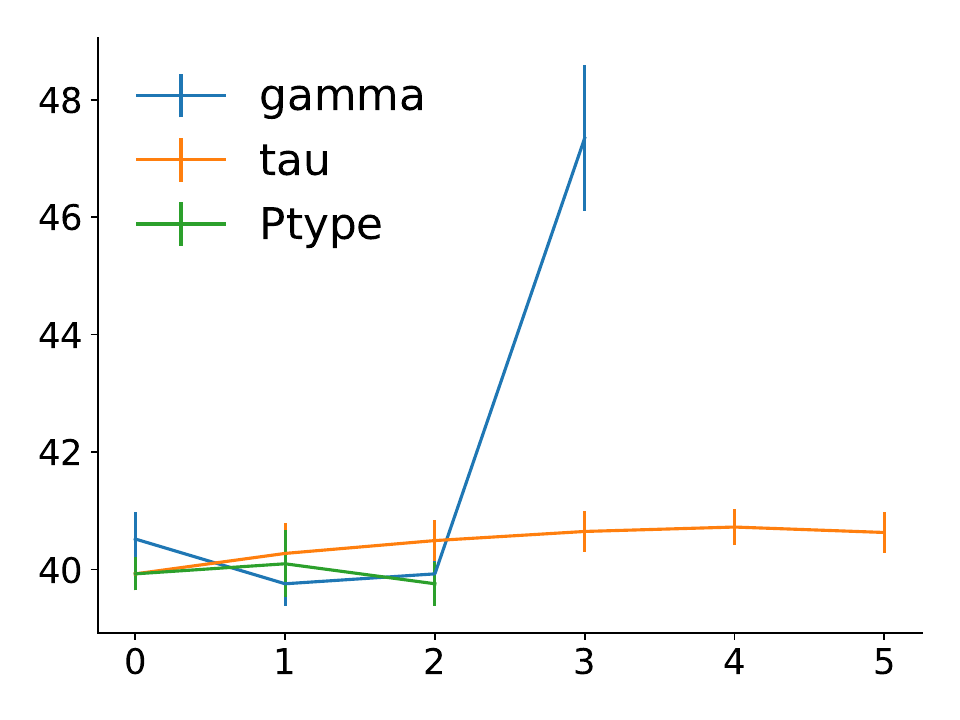}
    \label{fig:bikeshare}
  }
  \subfigure[ExecutionTime]{
    \includegraphics[width=\figwidththree]{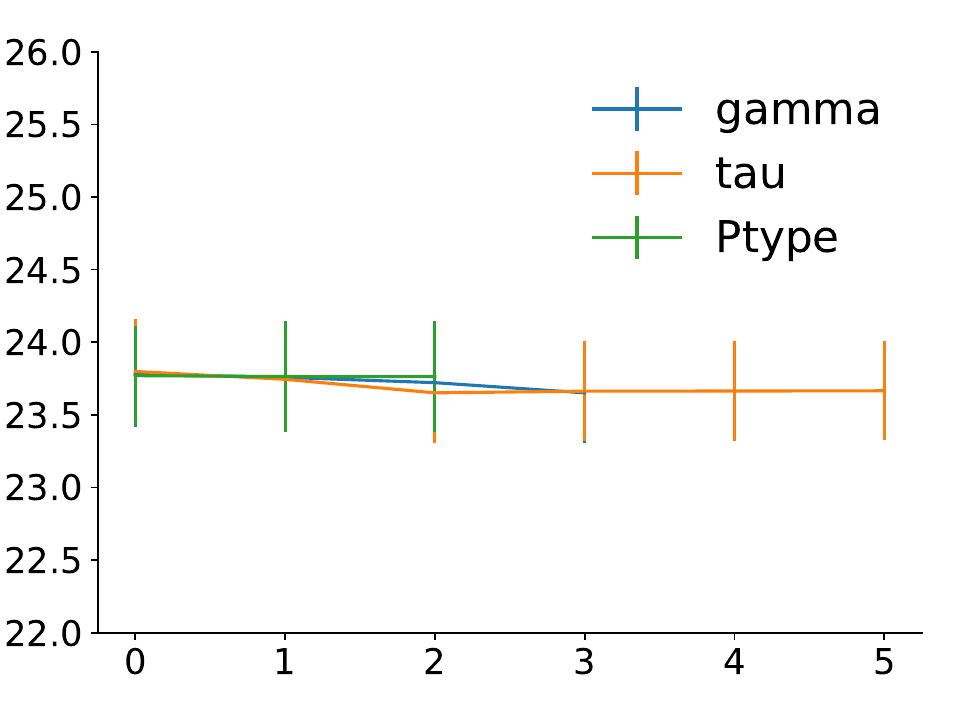}
    \label{fig:execsensi}
  }
  \subfigure[House]{
    \includegraphics[width=\figwidththree]{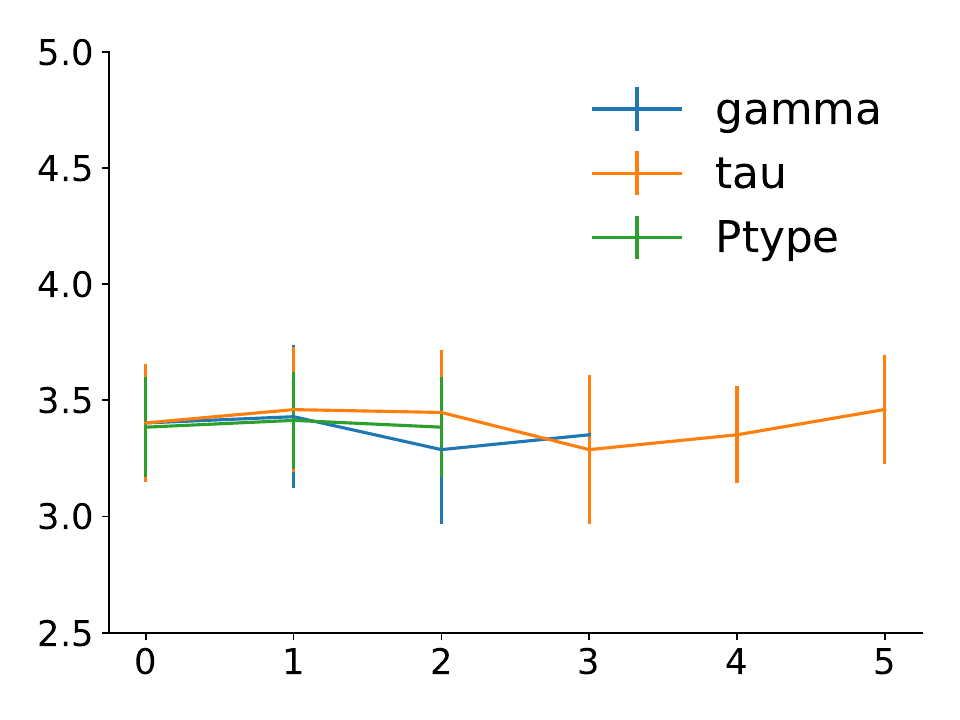}
    \label{fig:housesensi}
  }
  \caption{Sensitivity to hyperparameter settings. We show test RMSE with error bars while sweeping over discount rate $\gamma \in \{0.1, 0.2, 0.4, 0.9\}$, target NN moving rate $\tau \in \{0.001, 0.01, 0.1, 0.2, 0.4, 0.9\}$, and three types of transition probability. $\gamma$ hurts when it goes to the largest value on bikeshare, and it is likely because the exponential term in Poisson regression needs a unusually small learning rate. 
  When generating curves for $\tau$ and $\gamma$, we maintain $P_{\text{type}}$ as a simple uniform constant.
  }
  \label{fig:sensi}
\end{figure}

\section{Conclusion}
\label{sec:conclude}
This paper introduces a universal framework that transforms traditional SL problems into RL problems, alongside proposing a generalized TD algorithm. We highlight a specific set of problems for which such algorithms are apt and the potential variance reduction benefits of TD. 
Additionally, we theoretically formalize the iteration of our generalized TD learning algorithm as a special case of a generalized Bellman operator that is contractive and admits a unique fixed point. Based on this, we further establish the convergence properties of the algorithm.
Empirical evidence, encompassing both linear and deep learning contexts, is provided to validate our theoretical findings and to evaluate the algorithm's design and practical applicability. This research represents a foundational step towards bridging classic SL and RL paradigms. Our MRP formulation essentially offers a perspective that links various data points, corresponding to an interconnected worldview.

\textbf{Future work and limitations.} 
Our current analysis of covariance does not account for the potential advantages in scenarios where the matrix $\Smat$ is non-symmetric. This gap is due to both the absence of relevant literature and the lack of interpretable closed-form expressions, as the non-symmetric matrix can no longer be interpreted as a covariance matrix. It is plausible that specific types of transition matrix designs could be advantageous in certain applications -- particularly when they align with the underlying data assumptions. Additionally, our discussion does not encompass the potential benefits of incorporating more recent TD algorithms such as emphatic TD \citep{sutton2016etd}, gradient TD~\citep{maei2011gradient}, or quasi second-order accelerated TD~\citep{hengshuai2008td,pan2017atd,pan2017sketchatd}, which have been shown to enhance stability or convergence. Furthermore, it would be more intriguing to explore the utility of the probability transition matrix in a broader context, such as transfer learning, domain adaptation, or continual learning settings. Lastly, extending our approach to test its effectiveness with more modern NNs, such as transformers, could be of significant interest to a broader community.

\begin{acks}
Yangchen Pan would like to acknowledge the support from Tuing World leading fellow. 
Junfeng Wen would like to acknowledge the support from NSERC. 
Philip Torr is supported by the UKRI grant: Turing AI Fellowship. 
Philip Torr would also like to thank the Royal Academy of Engineering and FiveAI.
\end{acks}

\vskip 0.2in
\printbibliography


\appendix
The appendix is organized as follows. 
\begin{enumerate}
    \item \cref{sec:app-proofs} provides detail proofs for the theorems.
    \item \cref{sec:app-reproduce} provides implementation details on synthetic datasets from~\cref{sec:min_norm_equivalence}.
    \item \cref{sec:app-reproduce-real} provides implementation details and additional results for those real-world datasets:
    \begin{itemize}
        \item Linear regression (\cref{sec:app-additional-realworldlinear}),
        \item Regression with Neural Networks (NNs) (\cref{sec:app-additional-regnn-cor}),
        \item Binary classification with NNs (\cref{sec:app-additional-binary}),
        \item Implementation details of different heuristically designed $\Pmat$ are in~\cref{sec:app-designp}.
    \end{itemize}
\end{enumerate}

\section{Proofs}
\label{sec:app-proofs}
\subsection{Proof for Proposition~\ref{prop:equiv}}

\tdolsequiprop*

\begin{proof}
In our TD formulation, the reward is $\vr=(\Imat-\gamma \Pmat)\vy$. 
With $\Smat= \Dmat(\Imat-\gamma \Pmat)$,
we have $\Amat = \Xmat^\top \Smat \Xmat$, 
and $\vb = \Xmat^\top \Smat \vy$. 
To verify the first claim, define 
The min-norm solutions found by TD and OLS are respectively
$\vw_{TD} = \Amat^{\pinv}\vb$, 
and $\vw_{LS} = \Xmat^{\pinv}\vy$, where
\begin{small}
\begin{align}
\vw_{TD} =\Amat^\pinv \vb
= (\Xmat^\top \Smat \Xmat)^\pinv \cdot \Xmat^\top \Smat \vy.
\label{eq:TD_min_norm}
\end{align}
\end{small}%
When $\Dmat$ has full support, $\Smat$ is invertible (thus has linearly independent rows/columns). 
Additionally, $\Xmat$ has linearly independent rows so $(\Xmat^\top \Smat \Xmat)^\pinv=\Xmat^\pinv\Smat^{-1}(\Xmat^\top)^\pinv$ \citep[Thm.3]{greville1966note} and the TD solution becomes
\begin{align}
\vw_{TD} = \Amat^{\pinv}\vb 
= \Xmat^\pinv \Smat^{-1} (\Xmat^\top)^\pinv \Xmat^\top \Smat \vy
\end{align}
Finally, when $\Xmat$ has linearly independent rows, $(\Xmat^\top)^\pinv \Xmat^\top=\Imat_n$ so $\vw_{TD} = \Xmat^\pinv \vy=\vw_{LS}$. 

To verify the second part, the TD's linear system is 
\begin{small}
\begin{align}
\Xmat^\top \Smat \Xmat \wvec = \Xmat^\top \Smat \vy,
\end{align}
\end{small}%
which is essentially preconditioning the linear system $\Xmat \wvec = \yvec$ by $\Xmat^\top \Smat$. 
Hence, any solution to the latter is also a solution to TD.
\end{proof}

\subsection{Proof for Theorem~\ref{thm:expected_contraction}}

\lipschitzlemma*

\begin{proof}
\cref{assume:feature_regularity} and the compactness of $\Wcal$ ensure that the linear prediction $z=\vx(s)^\top\vw,\forall s\in\Scal$ will also be in a compact domain. 
Given that $f$ is continuous and invertible (\cref{assume:link_function}), effectively the domain and image of will also be compact. 
Furthermore, since $f'$ is bounded, there exists a constant $L=\max(L_f, L_{f^{-1}})\ge 1$ such that \cref{eq:bi_lipschitz} holds $\forall s_1,s_2\in\Scal$ with $z_1=\vx(s_1)^\top\vw,z_2=\vx(s_2)^\top\vw$, where $L_f,L_{f^{-1}}$ are the Lipschitz constants of 
$f$ and $f^{-1}$ respectively.
\end{proof}

We show the existence of an optimal solution $\vw^*\in\Wcal$ for our update using the Banach fixed-point theorem. 
In the following, we will show that any two parameters $\vw^A,\vw^B\in\Wcal$ will get closer after the one update step and a projection step.
To achieve this, we will first introduce two helpful lemmas in the following. 
To simplify notation, denote their predictions $z^A=\vx^\top \vw^A,z^B=\vx^\top \vw^B$, transformed predictions $y^A=f(z^A),y^B=f(z^B)$ and bootstrap targets $y_{td}^A=f(r+\vx'^\top\vw^A),y_{td}^B=f(r+\vx'^\top\vw^B)$.

\contractioncrossprod*

\begin{proof}
Note that
\begin{align}
(\vw^A-\vw^B)^\top [\overline{g}(\vw^B)-\overline{g}(\vw^A)]
&=(\vw^A-\vw^B)^\top 
\EE[\ [(y^B_{td}-y^B) - (y^A_{td}-y^A)]\cdot\vx\ ]
\\
&=\EE[\ (z^A-z^B)\cdot 
[(y^A-y^B) - (y^A_{td}-y^B_{td})]\ ]
\label{eq:cross_term_bound}
\end{align}
By \cref{lemma:lipschitz} and using the assumption that $f$ is strictly increasing, we have $(z^A-z^B)(y^A-y^B)\ge\frac{1}{L}(z^A-z^B)^2$.
Moreover, the function $\vz\mapsto f(\vr+\gamma \Pmat\vz)$ is ($\gamma L$)-Lipschitz so
\begin{align}
\EE\left[\ (z^A-z^B)(y_{td}^A-y^B_{td})\ \right]
\le \gamma L\cdot \EE[(z^A-z^B)^2].
\end{align}
Plugging these two into \cref{eq:cross_term_bound} completes the proof.
\end{proof}

\contractiongradnorm*

\begin{proof}
To start
\begin{align}
\|\overline{g}(\vw^A)-\overline{g}(\vw^B)\|_2
&=
\|\EE[\ [(y^A_{td}-y^A) - (y^B_{td}-y^B)]\cdot \vx\ ]\|_2
\\
&\le
\sqrt{\EE[\|\vx\|_2^2]}
\sqrt{\EE\left[
[(y^A_{td}-y^A) - (y^B_{td}-y^B)]^2
\right]}
&&\text{Cauchy-Schwartz}
\\
&\le
\sqrt{\EE\left[
[(y^B-y^A) - (y^B_{td}-y^A_{td})]^2
\right]}
&&\text{\cref{assume:feature_regularity}}
\end{align}
Let $\delta\defeq y^B-y^A$ and $\delta_{td}\defeq y^B_{td}-y^A_{td}$, then
\begin{align}
\EE\left[(\delta - \delta_{td})^2\right]
&=\EE[\delta^2]+ \EE[\delta_{td}^{2}]-2\EE[\delta\delta_{td}]
\\
&\le
\EE[\delta^2]+ \EE[\delta_{td}^{2}]+2 |\EE[\delta\delta_{td}]|
\\
&\le
\EE[\delta^2]+ \EE[\delta_{td}^{2}]+2 
\sqrt{\EE[\delta^2]\EE[\delta^{2}_{td}]}
&&\text{Cauchy-Schwartz}
\\
&=\left(
\sqrt{\EE[\delta^2]} + \sqrt{\EE[\delta^{2}_{td}]}
\right)^2
\end{align}
As a result,
\begin{align}
\|\overline{g}(\vw^A)-\overline{g}(\vw^B)\|_2
&\le
\sqrt{\EE\left[
(y^B-y^A)^2
\right]}
+
\sqrt{\EE\left[
(y^B_{td}-y^A_{td})^2
\right]}
\\
&\le L\left\{
\sqrt{\EE\left[
(z^B-z^A)^2
\right]}
+
\sqrt{\EE\left[
(z^B_{td}-z^A_{td})^2
\right]}
\right\}
\label{eq:grad_norm_bound}
\end{align}
Finally, note that
\begin{align}
\EE\left[(z_{td}^B - z_{td}^A)^2\right]
&=
\EE\left[\ [
(r+\gamma \vx'^{\top}\vw^B)
-(r+\gamma \vx'^{\top}\vw^A)
]^2\ \right]
\\
&=
\gamma^2 \EE\left[\ (\vx'^{\top}\vw^B
- \vx'^{\top}\vw^A)^2\ \right]
\\
&=
\gamma^2 \EE\left[\ (z^B - z^A)^2\ \right]
\end{align}
where the last line is because both $s,s'$ are assumed to be from the stationary distribution.
Plugging this to \cref{eq:grad_norm_bound} and using the fact that $\gamma\le 1$ complete the proof.
\end{proof}

Now we are ready to prove the contraction.
Given two iterates $\vw_t^A,\vw_t^B\in\Wcal$ and their updates
\begin{align}
\vw_{t+1}^A
=\Pcal(\vw^A_t+\alpha \overline{g}(\vw^A_t))
\qquad
\vw_{t+1}^B
=\Pcal(\vw^B_t+\alpha \overline{g}(\vw^B_t)),
\end{align}
the following theorem shows that the projected update is a contraction mapping.

\expectcontraction*

\begin{proof}
For $\vw^A_t,\vw^B_t\in\Wcal$, with probability one, for any $t\in\NN_0$
\begin{align}
\|\vw^A_{t+1}-\vw^B_{t+1}\|_2^2
&\le
\|\vw^A_{t}+\alpha \overline{g}(\vw^A_t)-(\vw^B_{t}+\alpha \overline{g}(\vw^B_t))\|_2^2
\end{align}
since projection onto a convex set is contracting.
Then we can further expand it as
\begin{align}
\|\vw^A_{t+1}-\vw^B_{t+1}\|_2^2
&\le
\|\vw^A_t-\vw_{t}^B\!\|_2^2
-2\alpha(\vw_t^A - \vw_{t}^B)^\top (\overline{g}(\vw_t^B)-\overline{g}(\vw_t^A))
\nonumber\\
&\qquad\qquad 
+\alpha^2\|\overline{g}(\vw_t^A)-\overline{g}(\vw_t^B)\|_2^2
\end{align}
where we can bound the second and third terms using \cref{lemma:cross_prod} and \cref{lemma:grad_norm} respectively so
\begin{align}
\|\vw^A_{t+1}-\vw^B_{t+1}\|_2^2
&\le 
\|\vw^A_t-\vw^B_{t}\|_2^2
-\left(
2\alpha\left(\frac{1}{L}-\gamma L\right)
-4L^2\alpha^2
\right)
\EE\left[ (z^A_t - z^B_t)^2 \right]
\end{align}
By choosing $\alpha=\frac{1-\gamma L^2}{4L^3}>0$
\begin{align}
\|\vw^A_{t+1}-\vw^B_{t+1}\|_2^2
&\le 
\|\vw^A_t-\vw^B_{t}\|_2^2
-\left(\frac{1-\gamma L^2}{2L^2}\right)^2
\EE\left[ (z^A_t - z^B_t)^2 \right]
\\
&= 
\|\vw^A_t-\vw^B_{t}\|_2^2
-\left(\frac{1-\gamma L^2}{2L^2}\right)^2
\EE\left[ (\vx^\top(\vw^A_t - \vw^B_t))^2 \right]
\\
&\le 
\|\vw^A_t-\vw^B_{t}\|_2^2
-\left(\frac{1-\gamma L^2}{2L^2}\right)^2
\EE\left[\|\vx\|_2^2\right] 
\|\vw^A_t - \vw^B_t\|_2^2
&&\text{Cauchy-Schwartz}
\\
&\le 
\|\vw^A_t-\vw^B_{t}\|_2^2
-\left(\frac{1-\gamma L^2}{2L^2}\right)^2 
\|\vw^A_t - \vw^B_t\|_2^2
&&\text{\cref{assume:feature_regularity}}
\end{align}
Finally, due to \cref{assume:bound_discount} and \cref{lemma:lipschitz}, we have
$1-\left(\frac{1-\gamma L^2}{2L^2}\right)^2\in[0,1)$
so it is a contraction mapping w.r.t. $\|\cdot\|_2$.
By the Banach fixed-point theorem,  the projected update admits a unique fixed point $\vw^*\in\Wcal$ and our iterate converges to it.
\end{proof}

\subsection{Proof for Theorem~\ref{thm:expected_update_rate}}

Once we know that the algorithm is convergent from \cref{thm:expected_contraction}, we are now ready to prove the convergence rate.

\expectconverge*

\begin{proof}
With probability 1, for any $t\in\NN_0$
\begin{align}
\|\vw^*-\vw_{t+1}\|_2^2
&=
\|\vw^*-\vw_{t}\|_2^2
-2\alpha(\vw^*-\vw_{t})^\top 
(\overline{g}(\vw_t) - \overline{g}(\vw^*))
+\alpha^2\|\overline{g}(\vw_t)-\overline{g}(\vw^*)\|_2^2
\\
&\le 
\|\vw^*-\vw_{t}\|_2^2
-\left(
2\alpha\left(\frac{1}{L}-\gamma L\right)
-4L^2\alpha^2
\right)
\EE\left[ (z^* - z_t)^2 \right]
\end{align}
where $z_t\defeq \vx^\top\vw_t$. 
The last inequality uses \cref{lemma:cross_prod} and \cref{lemma:grad_norm} with $\vw^A=\vw_t$ and $\vw_B=\vw^*$.
Using $\alpha=\frac{1-\gamma L^2}{4L^3}>0$
\begin{align}
\|\vw^*-\vw_{t+1}\|_2^2
&\le 
\|\vw^*-\vw_{t}\|_2^2
-\left(\frac{1-\gamma L^2}{2L^2}\right)^2
\EE\left[ (z^* - z_t)^2 \right]
\label{eq:consecutive_bound}
\end{align}
Telescoping sum gives
\begin{align}
\left(\frac{1-\gamma L^2}{2L^2}\right)^2
\times \sum_{t=0}^{T-1}
\EE\left[ (z^* - z_t)^2 \right]
\le 
\sum_{t=0}^{T-1}
(\|\vw^*-\vw_{t}\|_2^2
-\|\vw^*-\vw_{t+1}\|_2^2)
\le 
\|\vw^*-\vw_{0}\|_2^2
\end{align}
By Jensen's inequality
\begin{align}
\EE\left[ (z^* - \overline{z}_T)^2 \right]
\le
\frac{1}{T} \sum_{t=0}^{T-1}
\EE\left[ (z^* - z_t)^2 \right]
\le 
\left(\frac{2L^2}{1-\gamma L^2}\right)^2
\frac{\|\vw^*-\vw_{0}\|_2^2}{T}
\end{align}
Finally, since we assume that $\|\vx(s)\|_2^2\le 1,\forall s$, we have $\EE\left[ (z^* - z_t)^2 \right]\ge \omega\|\vw^*-\vw_{t}\|_2^2$ where $\omega$ is the maximum eigenvalue of the steady-state feature covariance matrix $\Sigmamat=\Xmat^\top \Dmat\Xmat=\sum_{s}D(s)\vx(s)\vx(s)^\top$.
Therefore, \cref{eq:consecutive_bound} leads to
\begin{align}
\|\vw^*-\vw_{t+1}\|_2^2
&\le 
\left(1-\omega
\left(\frac{1-\gamma L^2}{2L^2}\right)^2\right)
\|\vw^*-\vw_{t}\|_2^2
\\
&\le 
\exp\left(-\omega
\left(\frac{1-\gamma L^2}{2L^2}\right)^2\right)
\|\vw^*-\vw_{t}\|_2^2
&& \forall x\in\RR, 1-x\le e^{-x} 
\end{align}
Repeatedly applying this bound gives \cref{eq:exp_converge}.
\end{proof}

\subsection{Proof of Theorem~\ref{thm:sample_converge_rate}}

Here, we analyze the sample-based update (\ref{eq:sample_update}).
To account for randomness, let $\sigma^2\defeq \EE[\|g_{t}(\vw^*)\|_2^2]$, the variance of the TD update at the stationary point $\vw^*$ under the stationary distribution. 

\optimalitycondition*

\begin{proof}
Since $\vw^*$ is the fixed point of the iterative update (\ref{eq:expected_update}), it satisfies
\begin{align}
\vw^*=\argmin_{\vw\in\Wcal}\ \|\vw-(\vw^*+\alpha\overline{g}(\vw^*))\|_2^2.
\end{align}
This is a constrained optimization with a well-defined objective ($L_2$ distance), which means we can apply the first-order optimality condition~\citep[Sec.4.2.3]{boyd2004convex} and $\vw^*$ satisfies
\begin{align}
[\vw^*-(\vw^*+\alpha\overline{g}(\vw^*))]^\top
[\vw-\vw^*]\ge 0
\quad\forall \vw\in\Wcal.
\end{align}
Intuitively, this condition means that $\vw^*-(\vw^*+\alpha\overline{g}(\vw^*))$ makes a non-obtuse angle with any feasible direction $\vw-\vw^*$.
Simplifying this equation gives the desired result (\ref{eq:optimality_condition}).
\end{proof}

\gradnormsample*

\begin{proof}
To start
\begin{align}
\EE[\|g_t(\vw)\|_2^2]
&=
\EE[\|g_t(\vw)-g_t(\vw^*)+g_t(\vw^*)\|_2^2]
\\&\le
\EE[(\|g_t(\vw^*)\|_2+
\|g_t(\vw)-g_t(\vw^*)\|_2)^2]
&& \text{Triangle inequality}
\\&\le
2\EE[(\|g_t(\vw^*)\|_2^2]+
2\EE[\|g_t(\vw)-g_t(\vw^*)\|_2^2]
&& (a+b)^2 \le 2a^2 + 2b^2
\\&\le
2\sigma^2+
2\EE\left[\ 
\|[(y_{t,td}-y_t)-(y_{t,td}^*-y_t^*)]\vx_t\|_2^2
\ \right]
\\&\le
2\sigma^2+
2\EE\left[\ 
((y_{t,td}-y_t)-(y_{t,td}^*-y_t^*))^2
\ \right]
&& \text{\cref{assume:feature_regularity}}
\\&=
2\sigma^2+
2\EE\left[\ 
((y_t^*-y_t)-(y_{t,td}^*-y_{t,td}))^2
\ \right]
\\&\le
2\sigma^2+
4\left(\EE[(y_t^*-y_t)^2]+
\EE[(y_{t,td}^*-y_{t,td})^2]\right)
&& (a-b)^2\le 2a^2+2b^2
\label{eq:grad_norm_bound_sto}
\end{align}
where $y_t^*\defeq f(z_t^*)=f(\vx_t^\top\vw^*)$ and $y^*_{t,td}\defeq f(z^*_{t,td})=f(r_t+\vx_{t+1}^\top\vw^*)$. 
Note that by \cref{lemma:lipschitz}
\begin{align}
\EE[(y_t^*-y_t)^2]+
\EE[(y_{t,td}^*-y_{t,td})^2]
\le 
L^2\left\{
\EE\left[
(z_t^*-z_t)^2
\right]
+
\EE\left[
(z^*_{t,td}-z_{t,td})^2
\right]
\right\}.
\end{align}
Finally, 
\begin{align}
\EE\left[(z_{t,td}^* - z_{t,td})^2\right]
&=
\EE\left[\ [
(r_t+\gamma \vx_{t+1}^{\top}\vw^*)
-(r_t+\gamma \vx_{t+1}^{\top}\vw_t)
]^2\ \right]
\\
&=
\gamma^2 \EE\left[\ (\vx_{t+1}^{\top}\vw^*
- \vx_{t+1}^{\top}\vw_t)^2\ \right]
\\
&=
\gamma^2 \EE\left[\ (z_t^* - z_{t})^2\ \right]
\end{align}
where the last line is because both $s_t,s_{t+1}$ are from the stationary distribution. 
Combining these with \cref{eq:grad_norm_bound_sto} gives
\begin{align}
\EE[\|g_t(\vw)\|_2^2]
&\le 
2\sigma^2 + 4L^2(1+\gamma^2)
\EE\left[(z_t^*-z_t)^2\right]
\\&\le 
2\sigma^2 + 8L^2
\EE\left[(z_t^*-z_t)^2\right].
&& 0\le \gamma\le 1
\end{align}
\end{proof}

Now we are ready to present the convergence when using i.i.d.\ sample for the update:

\sampleconverge*

\begin{proof}

First note that for any $t\in\NN_0$
\begin{align}
\|\vw^*-\vw_{t+1}\|_2^2
&=
\|\Pcal(\vw^*)-
\Pcal(\vw_{t}+\alpha_t g_t(\vw_t))\|_2^2
\\&\le
\|\vw^*
-(\vw_{t}+\alpha_t g_t(\vw_t))\|_2^2
\end{align}
since $\vw^*\in\Wcal$ and projection onto a convex set is contracting.
Therefore,
\begin{align}
\EE[\|\vw^*-\vw_{t+1}\|_2^2]
&\le
\EE[\|\vw^*-\vw_{t}-\alpha_t g_t(\vw_t))\|_2^2]
\\&=
\EE[\|\vw^*-\vw_{t}\|_2^2]
-2\alpha_t\EE[(\vw^*-\vw_{t})^\top g_t(\vw_t)]
+\alpha_t^2\EE[\|g_t(\vw_t)\|_2^2]
\label{eq:sample_based_expansion}
\end{align}
Let's take a look at the second term more closely.
\begin{align}
\EE[(\vw^*-\vw_{t})^\top g_t(\vw_t)]
&=
\EE[(\vw^*-\vw_{t})^\top (g_t(\vw_t)-g_t(\vw^*))]
+
\EE[(\vw^*-\vw_{t})^\top g_t(\vw^*)]
\end{align}
and for the last term here we have
\begin{align}
\EE[(\vw^*-\vw_{t})^\top g_t(\vw^*)]
&=
\EE[\EE[(\vw^*-\vw_{t})^\top g_t(\vw^*)|\vw_t]]
\\&=
\EE[(\vw^*-\vw_{t})^\top \overline{g}(\vw^*)]
\ge0
\end{align}
which is due to \cref{lemma:optimality_condition}.
As a result, \cref{eq:sample_based_expansion} becomes
\begin{align}
\EE[\|\vw^*-\vw_{t+1}\|_2^2]
&\le
\EE[\|\vw^*-\vw_{t}\|_2^2]
\nonumber
\\&\qquad
-2\alpha_t\EE[(\vw^*-\vw_{t})^\top (g_t(\vw_t)-g_t(\vw^*))]
+\alpha_t^2\EE[\|g_t(\vw_t)\|_2^2]
\label{eq:sample_based_expansion_2}
\end{align}
Note that \cref{lemma:cross_prod} holds for any $\vw^A,\vw^B\in\Wcal$ and the expectation in $\overline{g}(\vw)=\EE[g_t(\vw)]$ is based on the sample $(s_t,r_t,s_{t+1})$, \emph{regardless of the choice of $\vw$}.
Thus, one can choose \emph{any} $\vw$ and then $\EE[g_t(\vw)|\vw]=\overline{g}(\vw)$. 
As a result, both \cref{lemma:cross_prod} and \cref{lemma:grad_norm_sample} can be applied to $\EE[(\vw^*-\vw_t)^\top (g_t(\vw_t)-g_t(\vw^*))|\vw_t]$ and $\EE[\|g_t(\vw_t)\|_2^2|\vw_t]$, respectively, in the following.
Thus \cref{eq:sample_based_expansion_2} becomes
\begin{align}
\EE[\|\vw^*-\vw_{t+1}\|_2^2]
&\le
\EE[\|\vw^*-\vw_{t}\|_2^2]
-2\alpha_t\EE[\EE[(\vw^*-\vw_{t})^\top (g_t(\vw_t)-g_t(\vw^*))|\vw_t]]
\nonumber
\\&\qquad
+\alpha_t^2\EE[\EE[\|g_t(\vw_t)\|_2^2|\vw_t]]
\\&\le 
\EE[\|\vw^*-\vw_{t}\|_2^2]
\nonumber
\\&\qquad
-\left(
2\alpha_t\left(\frac{1}{L}-\gamma L\right)
-8L^2\alpha_t^2
\right)
\EE\left[ (z^* - z_t)^2 \right]
+2\alpha_t^2\sigma^2
\\&\le
\EE[\|\vw^*-\vw_{t}\|_2^2]
-\alpha_t\left(\frac{1}{L}-\gamma L\right)
\EE\left[ (z^* - z_t)^2 \right]
+2\alpha_t^2\sigma^2
\end{align}
where the last inequality is due to $\alpha_t=\frac{1}{\sqrt{T}}\le\frac{1-\gamma L^2}{8L^3}$.
Then telescoping sum gives
\begin{align}
\frac{1}{\sqrt{T}}\left(\frac{1}{L}-\gamma L\right)
\sum_{t=0}^{T-1}\EE\left[ (z^* - z_t)^2 \right]
\le \|\vw^*-\vw_0\|_2^2 + 2\sigma^2
\\
\Longleftrightarrow
\sum_{t=0}^{T-1}\EE\left[ (z^* - z_t)^2 \right]
\le 
\frac{\sqrt{T} L}{1-\gamma L^2}
\left(
\|\vw^*-\vw_0\|_2^2 + 2\sigma^2
\right).
\end{align}
Finally, Jensen's inequality completes the proof
\begin{align}
\EE\left[ (z^* - \overline{z}_T)^2 \right]
\le 
\frac{1}{T}
\sum_{t=0}^{T-1}\EE\left[ (z^* - z_t)^2 \right]
\le 
\frac{L\left(
\|\vw^*-\vw_0\|_2^2 + 2\sigma^2
\right)}
{\sqrt{T}(1-\gamma L^2)}.
\end{align}
\end{proof}

\section{Experiment Details: Synthetic Data}
\label{sec:app-reproduce}

\label{sec:app-min-norm}
This subsection provides details of the empirical verification in \cref{sec:min_norm_equivalence}.

Each element of the input matrix $\Xmat$ is drawn from the standard normal distribution $\Ncal(0,1)$.
This (almost surely) guarantees that $\Xmat$ has linearly independent rows in the overparametrization regime (i.e., when $n<d$).
The true model $\vw^*$ is set to be a vector of all ones.
Each label $y_i$ is generated by $y_i=\vx_i^\top\vw^*+\epsilon_i$ with noise $\epsilon_i\sim\Ncal(0, 0.1^2)$.

We test various transition matrices $\Pmat$ as shown in \cref{tab:algo_relations}. 
For \texttt{Random}, each element of $\Pmat$ is drawn from the uniform distribution $U(0, 1)$ and then normalized so that each row sums to one. 
The \texttt{Deficient} variant is exact the same as \texttt{Random}, except that the last column is set to all zeros before row normalization. 
This ensures that the last state is never visited from any state, thus not having full support in its stationary distribution. 
\texttt{Uniform} simply means every element of $\Pmat$ is set to $1/n$ where $n=100$ is the number of training points.
\texttt{Distance (Close)} assigns higher transition probability to points closer to the current point, where the element in the $i$th row and the $j$th column is first set to $\exp(-(y_i-y_j)^2/2)$ then the whole matrix is row-normalized. 
Finally, \texttt{Distance (Far)} uses $1-\exp(-(y_i-y_j)^2/2)$ before normalization.
The last two variants are used to see if similarity between points can play a role in the transition when using our TD algorithm.

As shown in \cref{tab:algo_relations}, the min-norm solution $\vw_{TD}$ is very close to the min-norm solution of OLS as long as $n<d$ and $\Dmat$ has full support (non-deficient $\Pmat$). 
The choice of $\Pmat$ only has little effect in such cases. 
This synthetic experiment verifies our analysis in the main text.

\section{Experiment Details: Real-world Data}
\label{sec:app-reproduce-real}
\subsection{Implementation Details}
\label{sec:app-reallworld-commonsetup}

Deep learning experiments are based on tensorflow~\citep{tensorflow2015-whitepaper}, version 2.11.0, except that the ResNN18 experiments are using pytorch~\citep{paszke2017automatic}. Code is available at \url{https://github.com/yannickycpan/reproduceSL.git}. Below introduce common setup; different settings will be specified when mentioned. 

\textbf{Datasets.} We use three popular datasets house price~\citep{house}, execution time~\citep{exectime} and Bikeshare~\citep{fanaee2013bikedata} as benchmark datasets. We have performed one-hot encoding for all categorical variables and removed irrelevant features such as date and year as done by~\citep{pan2020implicit}. This preprocessing results in $114$ features. The Bikeshare dataset, which uses count numbers as its target variable, is popularly used for testing Poisson regressions. The air quality dataset~\citep{vito2016airquality} is loaded by using package \emph{ucimlrepo} by \texttt{from ucimlrepo import fetch\_ucirepo}.

For image datasets, we employ CNN consisting of three convolution layers with the number of kernels  $32$, $64$, and $64$, each with a filter size of $2\times 2$. This was followed by two fully connected hidden layers with $256$ and $128$ units, respectively, before the final output layer. 
On all image datasets, Adam optimizer is used and the learning rate sweeps over $\{0.003, 0.001, 0.0003\}$, $\gamma \in \{0.01, 0.1, 0.2\}$, $\tau \in \{0.01, 0.1\}$. The neural network is trained with mini-batch size $128$. For ResNN18, we use learning rate $0.001$ and $\tau=0.01$.

\textbf{Hyperparameter settings.} For regression and binary classification tasks, we employ neural networks with two hidden layers of size $256\times 256$ and ReLU activation functions. These networks are trained with a mini-batch size of 128 using the Adam optimizer~\citep{Kingma2015AdamAM}. In our TD algorithm, we perform hyperparameter sweeps for $\gamma \in \{0.1, 0.9\}$, target network moving rate $\tau \in \{0.01, 0.1\}$. For all algorithms we sweep learning rate $\alpha \in \{0.0003, 0.001, 0.003, 0.01\}$, except for cases where divergence occurs, such as in Poisson regression on the Bikeshare dataset, where we additionally sweep $\{0.00003, 0.0001\}$. Training iterations are set to $15$k for the house data, $25$k for Bikeshare, and $30$k for other datasets. We perform random splits of each dataset into $60\%$ for training and $40\%$ for testing for each random seed or independent run. Hyperparameters are optimized over the final $25\%$ of evaluations to avoid divergent settings. The reported results in the tables represent the average of the final two evaluations after smoothing window. 

\textbf{Naming rules.} For convenience, we repeat naming rules from the main body here. TDReg: our TD approach, with its direct competitor being Reg (conventional $l_2$ regression). 
Reg-WP: Utilizes the same probability transition matrix as TDReg but does not employ bootstrap targets.
This baseline can be used to assess the effect of bootstrap and transition probability matrix. On Bikesharedata, TDReg uses an exponential link function designed for handling counting data, and the baseline becomes Poisson regression correspondingly. 

\subsection{Additional Results on Linear Regression}
\label{sec:app-additional-realworldlinear}

As complementary results to the execute time dataset presented in ~\cref{sec:experiment}, we provide results on two other regression datasets below (see ~\cref{fig:bike-time-rho-P}). We consistently observe that the TD algorithm performs more closely to the underlying best estimator. Notably, the performance gain tends to increase as the correlation strengthens or as the transition probability matrix better aligns with the data correlation.

When implementing FGLS, we initially run OLS to obtain the residuals. Subsequently, an algorithm is employed to fit these residuals for estimating the noise covariance matrix. This matrix is then utilized to compute the closed-form solution. The implementation is done by API from~\citet{seabold2010statsmodels}.

In this set of experiments, to speed up multiple matrix inversion and noise sampling, we randomly take $500$ subset of the original datasets for training and testing. 

\begin{figure*}
  \centering
  \subfigure[bikeshare dataset]{
  \includegraphics[width=\textwidth]{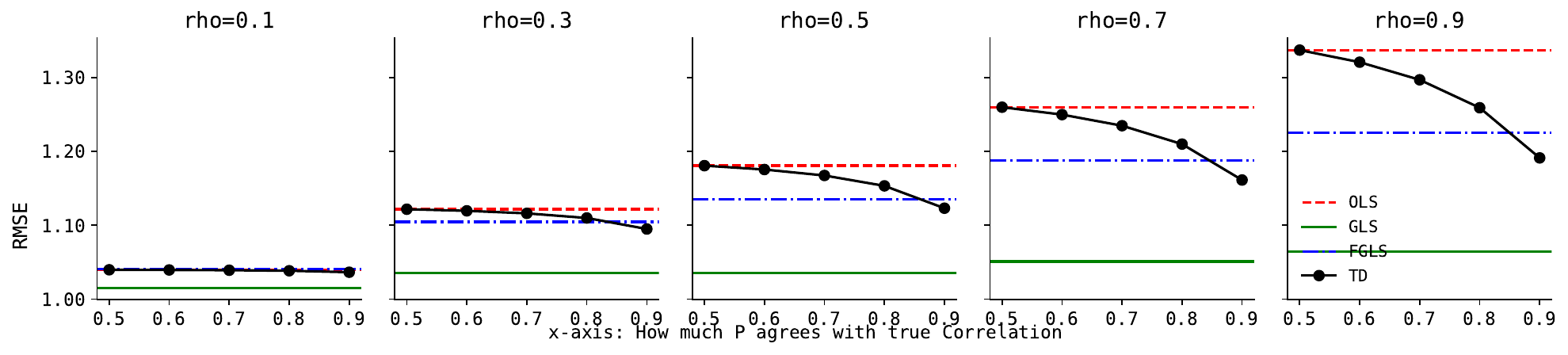}
  }
  \subfigure[house price dataset]{
  \includegraphics[width=\textwidth]{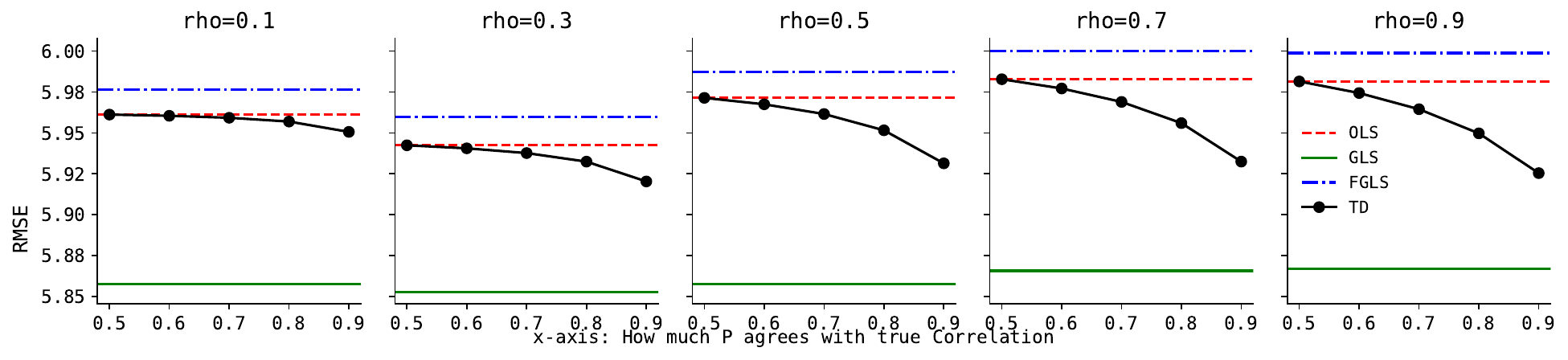}
  }
  \vspace{-0.4cm}
  \caption{\small Test Root Mean Squared Error (RMSE) versus different values of \( P \) on bikeshare and execution time dataset. From left to right, the noise correlation coefficient increases (\( \rho \in \{0.1, 0.3, 0.5, 0.7, 0.9\} \)). Each plot's x-axis represents the degree of alignment between the transition matrix \( P \) and the true covariance matrix that generates the noise. Consistent with our expectations, as \( P \) approaches the true correlation – implying a higher likelihood of data points with positively correlated noise transitioning from one to another – the solution derived from TD increasingly approximates the optimal one. Furthermore, as the correlation among the data intensifies, TD's solution is closer to optimum and one can see larger gap between TD and OLS/FGLS. The results are averaged over $30$ runs.}
  \label{fig:bike-time-rho-P}
\end{figure*}

\subsection{Additional Results: Regression with Correlated Noise with NNs}\label{sec:app-additional-regnn-cor}

Similar to that has been shown in ~\cref{fig:extimelc}, we show learning curves on increasingly strong correlated noise in ~\cref{fig:houselc}. In this set of experiments, to speed up noise sampling, we randomly take $1$k subset of the original datasets for training and testing. The covariance matrix used to generate correlated noise is specified in~\cref{sec:app-reproduce}.

\begin{figure}
  \centering
  \subfigure[House, $\sigma=10$]{
    \includegraphics[width=\figwidththree]{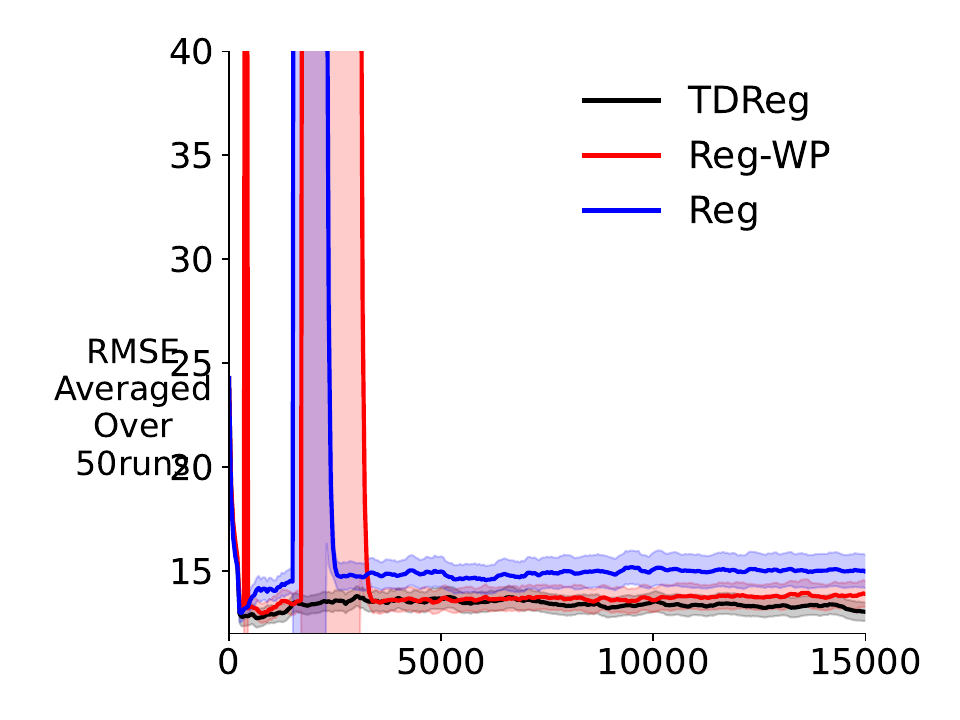}
  }
  \subfigure[House, $\sigma=20$]{
    \includegraphics[width=\figwidththree]{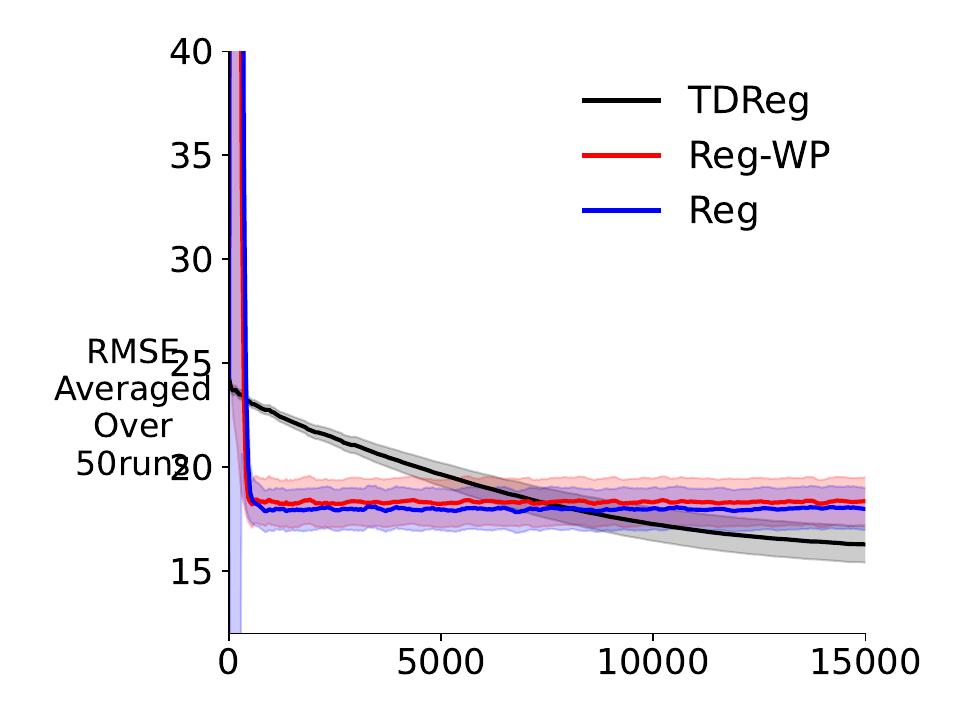}
  }
  \subfigure[House, $\sigma=30$]{
    \includegraphics[width=\figwidththree]{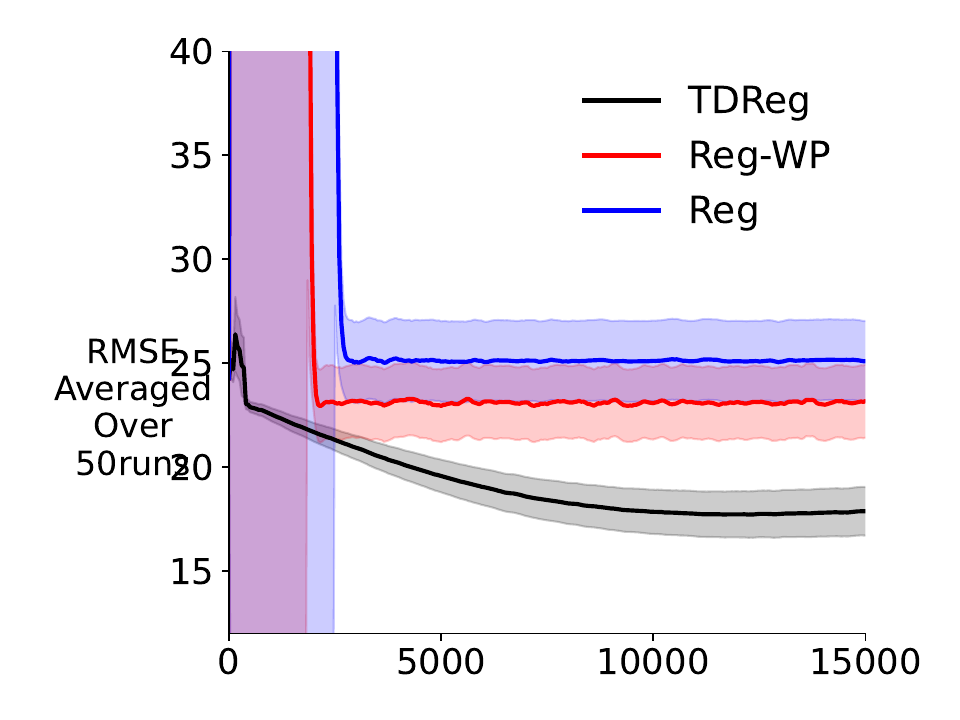}
  }
  \caption{Learning curves on house dataset. From left to right, the noise variance increases, with $\sigma\in {10, 20, 30}$. The results are averaged over 50 runs. Due to the addition of large-scale noise, the two baseline algorithms, Reg-WP and Reg, exhibit high instability during the early learning stages and ultimately converge to a suboptimal solution compared to our TD approach.
  }
  \label{fig:houselc}
\end{figure}

\subsection{Additional Results: Binary Classification with NNs}\label{sec:app-additional-binary}

For binary classification, we utilize datasets from Australian weather~\citep{ausweather}, 
and travel insurance~\citep{travelinsurance}. The aim of this series of experiments is to examine: 1) the impact of utilizing a link function; 2) with a specially designed transition probability matrix, the potential unique benefits of the TD algorithm in an imbalanced classification setting. Our findings indicate that: 1) the link function significantly influences performance; and 2) while the transition matrix proves beneficial in addressing class imbalance, this advantage seems to stem primarily from up/down-sampling rather than from the TD bootstrap targets.

Recall that we employ three intuitive types of transition matrices: $P(\xvec'|\xvec)$ is larger when the two points $\vx, \vx'$ 
are 1) similar (denoted as $P_s$); 
2) far apart ($P_f$); 
3) $P(\xvec'|\xvec)=1/n, \forall \xvec, \xvec' \in \mathcal{X}$  ($P_c$). 

Since our results in \cref{sec:experiments-standard} and Figure~\ref{fig:sensi} indicate that $\Pmat$ does not significantly impact regular regression, we conducted experiments on binary classification tasks and observed their particular utility when dealing with imbalanced labels. We define $P_s$ by defining the probability of transitioning to the same class as $0.9$ and to the other class as $0.1$. Table~\ref{tab:binary} presents the reweighted balanced results for three binary classification datasets with class imbalance. It is worth noting that in such cases, Classify-WP serves as both 1) no bootstrap baseline and 2) the upsampling techniques for addressing class imbalance in the literature~\citep{kubat2000imb}. 

Observing that TD-Classify and Classify-WP yield nearly identical results and Classify (without using TD's sampling) is significantly worse, suggesting that the benefit of TD arises from the sampling distribution rather than the bootstrap estimate in the imbalanced case. Furthermore, $P_f, P_s$ yield almost the same results in this scenario since they provide the same stationary distribution (equal weight to each class), so here Classify-WP represents both. We also conducted tests using $P_c$, which yielded results that are almost the same as Classify, and have been omitted from the table. In conclusion, the performance difference of TD in the imbalanced case arises from the transition probability matrix rather than the bootstrap target. The transition matrix's impact is due to the implied difference in the stationary distribution. 

\begin{table}[htbp]
\caption{\small Binary classification with imbalance. $0.0073$ means $0.73\%$ misclassification rate. The results are smoothed over $5$ evaluations before averaging over $5$ random seeds.}
\label{tab:binary}
\centering
\begin{tabular}{|c|c|c|c|c|}
\hline
\diagbox{Dataset}{Algs} & \textbf{TD-Classify} & \textbf{Classify-WP} & \textbf{Classify} & \textbf{TD-WOF} \\
\hline
\textbf{Insurance} & \textcolor{red}{$0.0073\pm0.0001$} & \textcolor{red}{$0.0073\pm0.0001$} & $0.0144\pm0.0003$ & $0.4994\pm0.0005$ \\
\hline
\textbf{Weather} & \textcolor{red}{$0.0695\pm0.0008$} & $0.0701\pm0.0008$ & $0.0913\pm0.0010$ & $0.4965\pm0.0031$ \\
\hline
\end{tabular}
\end{table}

\textbf{The usage of inverse link function}. The results of TD on classification without using a transformation/link function are presented in Table~\ref{tab:binary} and are marked by the suffix 'WOF.' These results are not surprising, as the bootstrap estimate can potentially disrupt the TD target entirely. 
Consider a simple example where a training example $\xvec$ has a label of one and transitions to another example, also labeled one. 
Then the reward ($r=y-\gamma y'$) will be $1-\gamma$. If the bootstrap estimate is negative, the TD target might become close to zero or even negative, contradicting the original training label of one significantly.

\begin{figure}
\centering
\subfigure[weather]{
\includegraphics[width=\figwidththree]{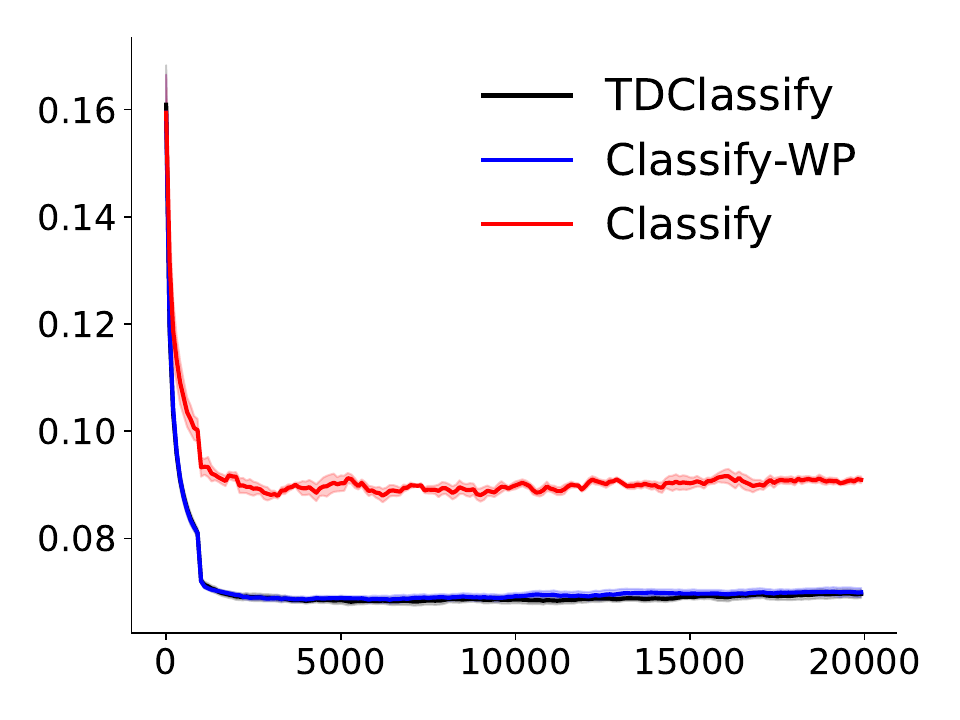}
}
\subfigure[insurance]{
\includegraphics[width=\figwidththree]{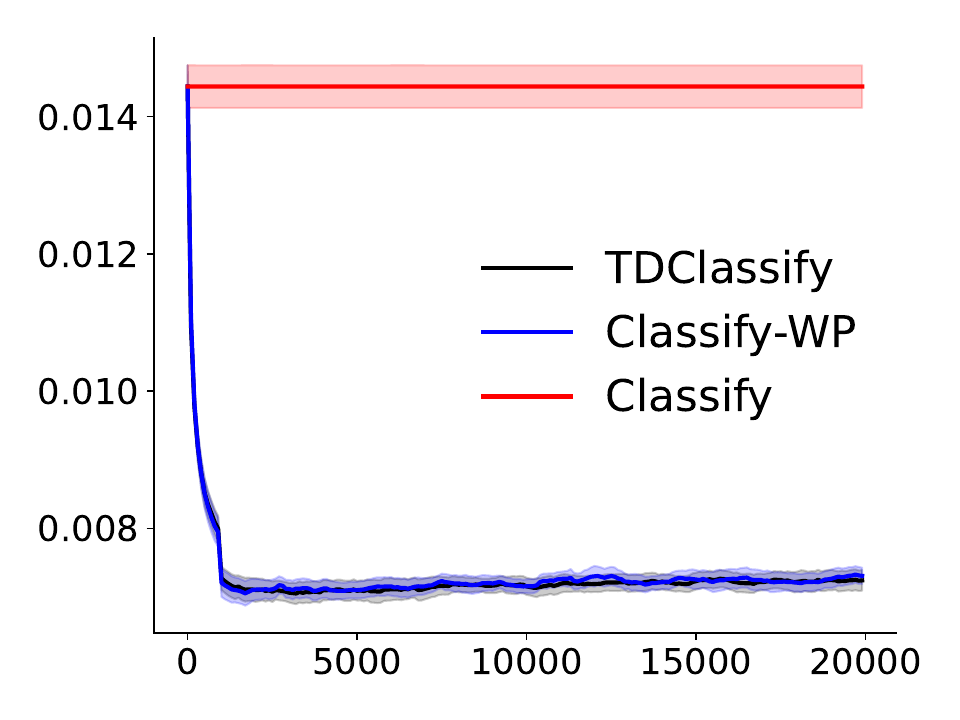}
}
\caption{Learning curves of binary classification with imbalance: balanced testing error v.s. training steps. The results have been smoothed using a $5$-point window before averaging over $5$ runs.}
\label{fig:binarylc}
\end{figure}

On binary classification dataset, weather and insurance, the imbalance ratios (proportion of zeros) are around $72.45\%, 88.86\%$ respectively. We set number of iterations to $20$k for both and it is sufficient to see convergence.
Additionally, in our TD algorithm, to prevent issues with inverse link functions, we add or subtract $1 \times 10^{-15}$ when applying them to values of $0$ or $1$ in classification tasks. It should be noted that this parameter can result in invalid values depending on concrete loss implementation, usually setting $<10^{-7}$ should be generally good. 

On those class-imbalanced datasets, when computing the reweighted testing error, we use the function from \url{https://scikit-learn.org/stable/modules/generated/sklearn.metrics.balanced_accuracy_score.html}. 

\subsection{On the Implementation of P}\label{sec:app-designp}

As we mentioned in Section~\ref{sec:experiment}, to investigate the effect of transition matrix, we implemented three types of transition matrices: $\Pmat(\xvec'|\xvec)$ is larger when the two points $x, x'$ 1) are similar; 2) when they are far apart; 3) $\Pmat(\xvec'|\xvec)=1/n, \forall \xvec, \xvec' \in \mathcal{X}$. To expedite computations, $P_s, P_f$ are computed based on the training targets instead of the features. The rationale for choosing these options is as follows: the first two may lead to a reduction in the variance of the bootstrap estimate if two consecutive points are positively or negatively correlated. 

The resulting matrix may not be a valid stochastic matrix, we use DSM projection~\citep{wang2010learning} to turn it into a valid one.

We now describe the implementation details. For first choice, given two points $(\xvec_1, y_1), (\xvec_2, y_2)$, the formulae to calculate the similarity is: 
\begin{equation}
    k(y_1, y_2) = \exp(-(y_1-y_2)^2/v) + 0.1
\end{equation}
where $v$ is the variance of all training targets divided by training set size. The second choice is simply $1- \exp(-(y_1-y_2)^2/v)$. 

Note that the constructed matrix may not be a valid probability transition matrix. To turn it into a valid stochastic matrix, we want: $\Pmat$ itself must be row-normalized (i.e., $\Pmat\onevec=\onevec$). To ensure fair comparison, we want equiprobable visitations for all nodes/points, that is, the stationary distribution is assumed to be uniform: $\vpi=\frac{1}{n}\onevec$.

The following proposition shows the necessary and sufficient conditions of the uniform stationary distribution property:

\begin{proposition}
$\vpi=\frac{1}{n}\onevec$ is the stationary distribution of an ergodic $\Pmat$ if and only if $\Pmat$ is a doubly stochastic matrix (DSM).
\end{proposition}

\begin{proof}
\textbf{If}:
Note that 
\begin{align}
\pi_j=\sum_{i} \pi_i p_{ij}
\quad\text{and}\quad
1=\sum_{i} p_{ij}
\end{align}
Subtracting these two gives 
\[
1-\pi_j=\sum_{i}(1-\pi_i)p_{ij}
\quad\text{and}\quad 
\frac{1-\pi_j}{n-1}=\sum_{i}\frac{1-\pi_i}{n-1}p_{ij}.
\]
This last equation indicates that $\left(\frac{1-\pi_1}{n-1},\frac{1-\pi_2}{n-1},\cdots,\frac{1-\pi_n}{n-1}\right)^\top$ is also the stationary distribution of $\Pmat$ (note that it is non-negative and sum to one). 
Due to the uniqueness of the stationary distribution, we must have
\[
\frac{1-\pi_i}{n-1}=\pi_i
\]
and thus $\pi_i=\frac{1}{n},\forall i$.

\textbf{Only if}:
Since $\vpi=\frac{1}{n}\onevec$ is the stationary distribution of $\Pmat$, we have
\[
\frac{1}{n}\onevec^\top \Pmat = \frac{1}{n}\onevec^\top
\]
and thus $\onevec^\top \Pmat = \onevec^\top$ indicating that $\Pmat$ is column-normalized.
Since $\Pmat$ is row-normalized by definition, it is doubly stochastic.
\end{proof}

With linear function approximation,
\begin{align}
\Amat &= \Xmat^\top \Dmat (\Imat-\gamma\lambda \Pmat)^{-1}
(\Imat-\gamma \Pmat) \Xmat
\\
\vb &= \Xmat^\top \Dmat (\Imat-\gamma\lambda \Pmat)^{-1}
(\Imat-\gamma \Pmat) \vy
\end{align}
where $\Xmat\in\RR^{n\times d}$ is the feature matrix, and $\Dmat$ is uniform when $\Pmat$ is a DSM. 
As a result, we can apply a DSM~\citep{wang2010learning} projection method to our similarity matrix.

\end{document}